\documentclass[10pt,twocolumn,letterpaper]{article}

\usepackage[accsupp]{axessibility}  %
\usepackage{iccv}
\usepackage{times}
\usepackage{epsfig}
\usepackage{graphicx}
\usepackage{amsmath}
\usepackage{amssymb}
\usepackage{array}
\usepackage{mydefs}
\usepackage{xcolor}
\usepackage{bm}
\usepackage{booktabs}
\usepackage{subcaption}
\usepackage{algorithm}
\usepackage{algpseudocode} 
\usepackage[normalem]{ulem}
\usepackage{multirow}
\usepackage{tabularx,rotating}
\usepackage{wrapfig}

\newcommand{\id}{\operatorname{Id}}

\usepackage[pagebackref=true,breaklinks=true,letterpaper=true,colorlinks,bookmarks=false]{hyperref}

\iccvfinalcopy %

\def\iccvPaperID{11126} %

\ificcvfinal\fi

\begin{document}

\title{Inverse problem regularization with hierarchical variational autoencoders\vspace{-0.2cm}}

\author{Jean Prost\\
Univ. Bordeaux, CNRS, Bordeaux INP\\
IMB, UMR 5251, F-33400 Talence, France\\
{\tt\small jean.prost@math.u-bordeaux.fr}
\and
Antoine Houdard\\
Ubisoft La Forge\\
Bordeaux\\
{\tt\small antoine.houdard@ubisoft.com}
\and
Andrés Almansa\\
Université Paris Cité, CNRS\\
 MAP5, UMR 8145\\
{\tt\small  andres.almansa@parisdescartes.fr}
\and 
Nicolas Papadakis\\
Univ. Bordeaux, CNRS, Bordeaux INP\\
IMB, UMR 5251, F-33400 Talence, France\\
{\tt \small  nicolas.papadakis@math.u-bordeaux.fr}
}

\maketitle
\ificcvfinal\thispagestyle{empty}\fi

\begin{abstract}
In this paper, we propose to regularize ill-posed inverse problems using a deep hierarchical variational autoencoder (HVAE) as an image prior. The proposed method synthesizes the advantages of i) denoiser-based Plug \& Play approaches and ii) generative model based approaches to inverse problems. First, we exploit VAE properties to design an efficient  algorithm that benefits from convergence guarantees of Plug-and-Play (PnP) methods. Second,  our approach is not restricted to specialized datasets and the proposed PnP-HVAE model is able to solve image restoration problems on natural images of any size. Our experiments show that the proposed PnP-HVAE method is competitive with both SOTA denoiser-based PnP approaches, and other SOTA restoration methods based on generative models. The code for this project is available at \href{https://github.com/jprost76/PnP-HVAE}{https://github.com/jprost76/PnP-HVAE}.
\vspace{-0.2cm}
\end{abstract}

\section{Introduction}
In this work, we study linear inverse problems %
\begin{equation}
    \label{eq:invprob}
 \y = A\x + \e
\end{equation}
in which  $\y \in \R^m$ is the degraded observation, $\x \in \R^d$ the original signal we wish to retrieve, $A\in \R^{m\times d}$ is an observation matrix
and $\e \sim \N(0, \sigma^2 I)$ is an additive Gaussian noise. 
Many image restoration tasks can be formulated as~\eqref{eq:invprob}, including deblurring, super-resolution or inpainting.
 
With the development of deep learning in computer vision, image restoration have known significant progress.
The most straight-forward way to exploit deep learning for solving image inverse problems is to train a neural network  to map degraded images to their clean version in a supervised fashion.
However, this type of approach requires a large amount of training data, and it lacks flexibility, as one network is needed for each different inverse problem.

An alternate approach is to use deep latent variable generative models such as GANs or VAEs and to compute the Maximum-a-Posterior (MAP) estimator in the latent space:
\begin{equation}
    \label{eq:map}
    \hat{\z} = \arg\max_{\z} \log p\left(\y|G(\z)\right) + \log p\left(\z\right),
\end{equation}
where $\z$ is the latent variable and $G$ is the generative network~\cite{bora2017compressed,menon2020pulse}.
In~\eqref{eq:map} the likelihood $p\left(\y|G(\z)\right)$ is related to the forward model~\eqref{eq:invprob}, and $p\left(\z\right)$ corresponds to the prior distribution over the latent space.
After solving~\eqref{eq:map}, the solution of the inverse problem is defined as $\hat{\x}=G\left(\hat{\z}\right)$.
The latent optimization methods~\eqref{eq:map} provide high-quality solutions that are guaranteed to be in the range of a generative network. %
However, this implies highly non-convex problems~\eqref{eq:map} due to the complexity of the generator %
and the obtained  solutions  may lack of consistency with the degraded observation~\cite{Saharia2021}.
Although  the convergence of latent optimization algorithms has been studied in the  literature, existing convergence guarantees are either restricted to specific settings, or rely on assumptions that are hard to verify.

In this work, we propose an algorithm that exploits the strong prior of a deep generative model while providing realistic convergence guarantees.
We  consider a specific type of deep generative model, the hierarchical variational autoencoder (HVAE).
 HVAE gives state-of-the-art results on image generation benchmarks~\cite{vahdat2020nvae, child2021very,hazami2022efficient,luhman2022optimizing}, and provides an encoder that will be key in the design of our proposed method.

As the HVAE  model differs significantly from the architecture of concurrent models, it is necessary to design algorithms adapted to their specific structure.
The latent space dimension of HVAE is significantly higher than the image dimension.
 Hence, constraining the solution to lie in the image of the generator is not sufficient enough to regularize inverse problems.
Indeed, it has been observed that HVAEs can perfectly reconstruct out-of-domain images~\cite{havtorn2021hierarchical}. 
Consequently, we propose to constrain the latent variable of the solution to lie in the high probability area of the HVAE prior distribution.
This can be done efficiently by controlling the variance of the prior over the latent variables.

The common practice of optimizing the latent variables of the generative model with backpropagation is impractical due to the high dimensionality of the hierarchical latent space. Instead, we exploit the HVAE encoder to define an alternating algorithm~\cite{gonzalez2022solving} to optimize the joint distribution over the image and its latent variable

To derive convergence guarantees for our algorithm, we show that it can be reformulated as a Plug-and-Play (PnP) method~\cite{venkatakrishnan2013plug}, which alternates between an application of the proximal operator of the data-fidelity term, and a reconstruction by the HVAE.
Under this perspective, we give sufficient conditions to ensure the convergence of our method, and we  provide an explicit characterization of the fixed-point
of the iterations. 
Motivated by the parallel with PnP methods, we name our method PnP-HVAE.

\subsection{Contributions and outline}
In this work, we introduce PnP-HVAE, a method for regularizing 
image restoration problems with a hierarchical variational autoencoder.
Our approach exploits the expressiveness of a deep HVAE generative model and its capacity to provide a strong prior on specialized datasets,  as well as convergence guarantees of Plug-and-Play  methods and their ability to 
deal with natural images of any size. 
After a review of related works (section~\ref{sec:related}) and of the  background on HVAEs (section~\ref{sec:back}), our contributions are the following. 
\newline\noindent$\bullet$   In section~\ref{sec:alternate}, we introduce PnP-HVAE, an algorithm to solve inverse problems with a HVAE prior.  
PnP-HVAE optimizes a joint posterior on image and latent variables without backpropagation through the generative network. It 
can be viewed as a generalization of JPMAP~\cite{gonzalez2022solving} to hierarchical VAEs,  with additional control of the regularization.
\newline\noindent$\bullet$   In section~\ref{sec:convergence}, we demonstrate the convergence of PnP-HVAE under hypotheses on the autoencoder reconstruction. Numerical experiments illustrate that the technical hypotheses are empirically met on noisy images with our proposed architecture. We also  exhibit the  better convergence properties of our alternate algorithm with respect to the use of Adam for optimizing the joint posterior objective.
\newline\noindent$\bullet$   In section~\ref{sec:expe}, we demonstrate the effectiveness of PnP-HVAE through image restoration experiments and comparisons on
(i) faces images using the pre-trained VDVAE model from~\cite{child2021very}; and (ii) natural images using the proposed PatchVDVAE architecture trained on natural image patches. %

\section{Related works}\label{sec:related}
CNN methods for regularizing inverse problems can be classified in two categories: regularization with denoisers (Plug-and-Play) and regularization with generative models. %

\subsection{Plug-and-Play methods} 

Plug-and-Play (PnP) and RED methods~\cite{venkatakrishnan2013plug,romano2017little} make use of a (deep) denoising method as a proxy to encode the local information over the prior distribution.
The denoiser is plugged in an optimization algorithm such as Half-Quadratic Splitting or ADMM in order to solve the inverse problem.
PnP algorithms come with theoretical convergence guarantees by imposing certain conditions on the denoiser network~\cite{ryu2019plug, pesquet2021learning, hurault2022gradient}.
These approaches provide state-of-the-art results on a wide variety of image modality thanks to the excellent performance of the currently available deep denoiser architectures~\cite{zhang2021plug}. 
However, PnP methods are only implicitly related to a probabilistic model, and they provide limited performance for challenging structured problems such as the inpainting of large occlusions.

\subsection{Deep generative models for inverse problems}
Generative models represent an explicit  image prior that can  be used to regularize ill-posed inverse problems~\cite{bora2017compressed, latorre2019fast, menon2020pulse, daras2021intermediate,  oberlin2021regularizing, pan2021exploiting,song2021solving}. 
They are latent variable models parametrized by neural networks, optimized to fit a training data distribution~\cite{kingma2013auto, goodfellow2020generative, dinh2016density, ho2020denoising}.

\noindent
{\bf  Convergence issues.} 
Regularization with generative models~\eqref{eq:map} %
 involves highly non-convex optimization over  latent variables \cite{bora2017compressed,menon2020pulse,oberlin2021regularizing}. %
The convergence guarantees of existing methods remain to be established~\cite{shah2018solving,raj2019gan,holden2022bayesian}. %

Convergent methods have only been proposed for restricted uses cases.
In compressed sensing problems with Gaussian measurement matrices, one can show that the objective function has a few critical points and design an algorithm to find the global optimum~\cite{hand2018phase,huang2021provably}. With a prior given by a VAE, and under technical hypothesis on the encoder and the VAE architecture, the Joint Posterior Maximization  with Autoencoding Prior (JPMAP) framework of~\cite{gonzalez2022solving} %
 converges toward a minimizer of the joint posterior $p(\x,\z|\y)$  of the image $\x$ and latent $\z$  given the  observation $\y$. 
 JPMAP is nevertheless only designed for VAEs of limited expressiveness, with a simple  fixed Gaussian prior distribution over the latent space. This makes it impossible to use this approach for anything other than toy examples.

\noindent
{\bf  Genericity issues.} When a highly structured dataset with fixed image size  is available ({\em e.g.}  face images~\cite{karras2019style}), deep generative models %
produce image restorations of impressive quality for severely ill-posed problems, such as %
super-resolution with huge upscaling factors.

For natural images, deep priors have \hbox{been improved by} %
normalizing flows \cite{rezende2015variational,dinh2016density}, GANs~\cite{goodfellow2020generative},  score-based  \cite{song2020score} and diffusion models \cite{ho2020denoising,Saharia2021,Saharia2022}. Note that the  use of  diffusion models in PnP is still based on  approximations~\cite{Song2023-pigdm,chung2022diffusion}, assumptions~\cite{Kawar2022,Meng2022} or empirical algorithms~\cite{Lugmayr2022}. 

While GAN models were considered SOTA, modern  hierarchical VAE (HVAE) architectures were shown to display  quality on par with GANs \cite{child2021very,vahdat2020nvae}, %
while being 
 able to perfectly reconstruct out-of-domain images~\cite{havtorn2021hierarchical}.

Integrating HVAE in a convergent scheme for natural image restoration of any size raises several theoretical and methodological challenges, as the image model of HVAE is  the push-forward of a {\em causal cascade of latent distributions}.

\section{Background on  variational autoencoders}\label{sec:back}
This work  exploits the capacity of HVAEs to model complex image distributions~\cite{child2021very,vahdat2020nvae}.
We %
 review the properties of the loss function used to train a VAE
(section~\ref{sec:perfect_vae}), then we detail the generalization of VAE to hierarchical VAE (section~\ref{ssec:hvae}), and present the temperature scaling approach to monitor the quality of generated images (section~\ref{ssec:temp}).

\subsection{VAE training} \label{sec:perfect_vae}

Variational autoencoders (VAE)  have been introduced in~\cite{kingma2013auto} to model complex data distributions. VAEs are trained to fit a parametric probability distribution in the form of a latent variable model:
\begin{equation}
    \pt{\x} = \int{\pt{\x|\z}\pt{\z}d\z},
\end{equation}
where $\pt{\x|\z}$ is a probabilistic decoder, and $\pt{\z}$ corresponds to the prior distribution over the model latent variable $\z$.
A VAE is also composed of a probabilistic encoder $\qp{\z|\x}$, whose role is to approximate the posterior of the latent model $\pt{\z|\x}$, which is usually intractable.

The generative model parameters $\theta\in\Theta$ and the approximate posterior parameters $\phi\in\Phi$ of a VAE are jointly trained by maximizing the evidence lower bound (ELBO)~\cite{kingma2013auto, rezende2014stochastic} on a training data distribution $\pd{\x}$.
\begin{equation}
    \label{eq:elbo}
    \mathcal{L}\left(\x; \theta, \phi\right)\hspace{-1pt} =\hspace{-1pt} \E{\qp{\z|\x}}{\log\pt{\x|\z}} - \KL{\qp{\z|\x}}{\pt{\z}}.
\end{equation}
The ELBO expectation on $\pd{\x}$ is upper-bounded by the negative entropy of the data distribution,
and, when the upper-bound is reached we have that~\cite{zhao2017learning}:
\begin{equation}
    \label{eq:perfect_vae}
    \KL{\pd{\x}\qp{\z|\x}}{\pt{\x}\pt{\z|\x}} = 0.
\end{equation}

\subsection{Hierarchical variational autoencoders}\label{ssec:hvae}

The ability of hierarchical VAEs to model complex distributions is due to their hierarchical structure imposed in the latent space.
 The latent variable of HVAE is partitioned into $L$ subgroups $\z = (\z_0, \z_1, \cdots, \z_{L-1})$, and the prior and the encoder are respectively defined as: 
\begin{align}
    \pt{\z} &= \prod_{l=1}^{L-1} \pt{\zl|\z_{<l}}\pt{\z_0}  \label{eq:hierarchical_prior}\\
    \qp{\z|\x} &= \prod_{l=1}^{L-1} \qp{\zl|\z_{<l},\x}\qp{\z_0|\x},    \label{eq:hierarchical_encoder}
\end{align}
We consider a specific class of HVAEs with Gaussian conditional distributions for the encoder and the decoder
\begin{equation}
    \label{eq:gaussian_hvae}
    \hspace*{-0.25cm}\begin{cases}
        \pt{\zl|\z_{<l}} &= \mathcal{N}\!\left(\zl; \mutl{\z_{<l}}\hspace{-1pt}\hspace{-1pt}\Sigma_{\theta, l}(\z_{<l})\right) \\
        \qp{\zl|\z_{<l}, \x}\mkern-18mu &= \mathcal{N}\!\left(\zl; \mupl{\z_{<l}, \x}\hspace{-1pt},\hspace{-1pt} \Sigma_{\phi, l}(\z_{<l}, \x)\right)\hspace{-1pt},
    \end{cases}
\end{equation} 
where $\mu_{\theta, 0}$ and $\Sigma_{\theta, 0}$ can either be trainable or non-trainable constants, and the remaining mean vectors ($\mu_{\theta,l}$, and $\mu_{\phi,l}$, for $l>0$) and covariance matrices ($\Sigma_{\theta,l}$ and $\Sigma_{\phi,l}$, for $l>0$)  are parametrized by neural networks. \footnote{
	Note that for the special case $l=0$, $\z_{<l}$ is empty, meaning that 
	$\Sigma_{\theta, 0}(\z_{<0})=\Sigma_{\theta, 0}$ is actually a constant,
	$\qp{\z_0|\z_{<0},\x}=\qp{\z_0|\x}$ is only conditioned on $\x$, etc.
}
In this work, we consider models with a Gaussian decoder:
\begin{equation}
    \label{eq:decoder_constant}
    \pt{\x|\z} = \N\left(\x; \mut{\z}, \gamma^2 I\right).
\end{equation}

\subsection{Temperature scaling}\label{ssec:temp}
As demonstrated in~\cite{vahdat2020nvae,child2021very}, sampling the latent variables $\zl$ from a prior with reduced temperature improves the visual quality of the generated images from the VAE.
In practice, this is done by multipling the covariance matrix of the Gaussian distribution $\pt{\zl|\z_{<l}}$ by a factor $\tau_l<1$.
This factor $\tau_l$ is called temperature because of its link to statistical physics.
Reducing the temperature of the priors amounts to defining the auxilliary model:
\begin{equation}
    \label{eq:hvae_joint_model_temp}
    p_{\theta, \tauvec}\left(\z_0, \cdots, \z_{L-1}, \x\right) = \prod_{l=0}^{L-1} \frac{\pt{\zl|\z_{<l}}^{\frac{1}{\tau_l^2}}}{Z_l}\pt{\x|\z_{<L}},
\end{equation}
where $\tauvec := \left(\tau_0, \cdots, \tau_{L-1}\right)$ gives the temperature for each level of the  hierarchy, 
and the $Z_l$ are normalizing constants.
In the following, we use this  temperature-scaled model to balance the regularization of our inverse problem.

\section{Regularization with HVAE Prior}
\label{sec:alternate}
In this section  we introduce our  Plug-and-Play method using a Hierarchical VAE  prior (PnP-HVAE) to solve generic image inverse problems. Building on top of the JPMAP framework~\cite{gonzalez2022solving}, we propose a joint model  over  the image and its latent variable that we optimize in an alternate way. By doing so, we take advantage of the HVAE encoder to avoid backpropagation through the generative network.
Although our motivation is similar to JPMAP, PnP-HVAE overcomes two of its limitations, that are the lack of control, and the limitation to simple and non-hierarchical VAEs.
We show in section~\ref{ssec:HJP} that the strength of the regularization of the tackled inverse problem can be monitored by tuning the temperature of the prior in the latent space. In section~\ref{ssec:encoder}, we  propose an approximation of the joint posterior distribution using the hierarchical VAE encoder. 
In section~\ref{ssec:optim}, we present our final algorithm that includes a new greedy scheme to optimize the latent variable of the HVAE.

\subsection{Tempered hierarchical joint posterior}\label{ssec:HJP}
The linear image degradation model~\eqref{eq:invprob} yields %
\begin{equation}
    \label{eq:f}
    p(\y|\x) \propto e^{-f(\x)}, \quad f(\x) = \frac{1}{2\sigma^2}||A\x-\y||^2.
\end{equation}
Solving the underlying image inverse problem in a bayesian framework requires an a priori distribution $p(\x)$ over clean images and studying the posterior distribution of $\x$ knowing its degraded observation $\y$.
In this work, the image prior is given by a hierarchical VAE and we exploit the joint posterior model $p(\z,\x|\y)$. 
From the HVAE latent variable model with reduced temperature $p_{\theta, \bf{\tau}}\left(\z_0, \cdots, \z_{l-1}, \x\right)$ as defined in \eqref{eq:hvae_joint_model_temp},
we define the associated tempered joint model as:
\begin{equation}
    \label{eq:full_joint_model}
    p\left(\z%
    , \x, \y\right) 
    := p_{\theta, \bf{\tau}}\left(\z_0, \cdots, \z_{L-1}, \x\right)p\left(\y|\x\right).
\end{equation}
Following the JPMAP idea from~\cite{gonzalez2022solving}, we aim at finding the couple $(\x,\z)$ that maximizes the joint posterior $p(\x, \z|\y)$: %
\begin{equation}
    \label{eq:joint_map}
    \min_{\x, \z} - \log p(\x, \z|\y) .
\end{equation}
Although we are only interested in finding the image $\x$,  the joint Maximum A Posteriori (MAP) criterion \eqref{eq:joint_map} makes it possible to derive an optimization scheme that only relies on forward calls of the HVAE. %
Using  Bayes' rule and the definition of the tempered HVAE joint model~\eqref{eq:full_joint_model} and~\eqref{eq:hvae_joint_model_temp},
the logarithm of the joint posterior rewrites:
\begin{align}
    &\log p(\x, \z|\y)+\log p(\y) \\ =&\log p(\y|\x)  + \sum_{l=0}^{L-1} \log\frac{\pt{\zl|\z_{<l}}^{\frac{1}{\tau_l^2}}}{Z_l} + \log \pt{\x|\z_{<L}}\nonumber.
\end{align}

Since $p(\y)$ is constant, 
finding the joint MAP estimate \eqref{eq:joint_map} amounts to minimizing the following criterion:
\begin{align}\nonumber
    J_1\left(\x, \z\right):=&- \sum_{l=1}^{L-1}\frac{1}{\tau_l^2}\log \pt{\zl|\z_{<l}}\\
    &+f(\x) -\log \pt{\x|\z_{<L}}.
    \label{eq:J1}
\end{align}
Notice that the temperature of the prior over the latent space $\tau_l$ controls the weight of the regularization over the latent variable $\zl$.
Optimizing~\eqref{eq:J1} w.r.t. $\x$ is tractable, whereas the minimization  w.r.t.$\z$ requires a backpropagation through the decoder $\log \pt{\x|\z_{<L}}$ that is impractical due to the high dimensionality and the hierarchical structure of the HVAE latent space.

\subsection{Encoder approximation of the joint posterior}\label{ssec:encoder}
Using the encoder $q_\phi$, we can reformulate the joint MAP problem~\eqref{eq:J1} in a form that is more  convenient to optimize with respect to $\z$.
Indeed, assuming that the encoder perfectly matches the true posterior, we have that:
\begin{equation}
    \pt{\x|\z} = \frac{\qp{\z|\x}\pd{\x}}{\pt{\z}} \label{eq:decoder_approx}.
\end{equation}
This assumption can be met if the variational family $\{\qp{.|\x}; \phi \in \Phi\}$ contains the  true posterior $p(\z|\x)$ and is trained to optimality, following equation~\eqref{eq:perfect_vae}. If this assumption appears unrealistic for vanilla (non-hierarchical) VAE \cite{gonzalez2022solving}, our experiments suggest that HVAE is sufficiently expressive to match the posterior to a reasonably good accuracy.

Thus, by introducing the decoder expression~\eqref{eq:decoder_approx} in the full model~\eqref{eq:full_joint_model}, we have:
\begin{align}
    \nonumber
    p\left(\z, \x, \y\right)=
    \prod_{l=0}^{L-1} \frac{1}{Z_l}\frac{\qp{\zl|\z_{<l}, \x}}{\pt{\zl|\z_{<l}}^{1-\tau_l^{-2}}}p(\y|\x)\pd{\x}.
\end{align} 
Denoting $\lambda_l = \frac{1}{\tau_l^2}-1$, we reformulate the joint MAP problem~\eqref{eq:joint_map} as a joint MAP problem over the encoder model:
\begin{align}
    \nonumber
        \min_{\x,\z}\hspace{-1pt}J_2(\x, \z)\hspace{-2pt} :=&  \hspace{-1pt}-\hspace{-4pt}\sum_{l=0}^{L-1} \hspace{-3pt}\left(\log\qp{\zl|\x, \z_{<l}} \hspace{-2pt}+\hspace{-2pt} \lambda_l\hspace{-1pt}\log\pt{\zl|\z_{<l}}\hspace{-1pt}\right)\\&+f(\x)- \log\pd{\x}\label{eq:J2}.
\end{align}

\subsection{Alternate optimization with PnP-HVAE}\label{ssec:optim}
We introduce an alternate  scheme to minimize~\eqref{eq:joint_map} that  sequentially optimizes with respect to $\x$ and  to $\z$.
For a linear degradation model and a Gaussian decoder~\eqref{eq:invprob}, the criterion $J_1(\x, \z)$ in~\eqref{eq:J1} is convex in $\x$ and its global minimum is
$\x = \left(A^tA+\frac{\sigma^2}{\gamma^2}\id\right)^{-1}
        \hspace{-4pt}\left(
            A^t\y+\frac{\sigma^2}{\gamma^2}\mut{\z}\right)$. %
Next we propose to compute an approximate solution of the problem $\min_{\z} J_2(\x,\z)$ with the greedy %
 algorithm~\ref{algo:hierarchical_latent_reg}.
\begin{algorithm}
    \caption{Hierarchical encoding with latent regularization to minimize~\eqref{eq:J2} w.r.t. $\z$ for a fixed $\x$}\label{algo:hierarchical_latent_reg}
        \begin{algorithmic}\small
            \Require image $\x$; HVAE ($\phi,\theta$); temperature $\tau_l$;  $\lambda_l=\frac{1}{\tau_l^2}-1$

        \For{$0\leq l < L$}
      \State $S_q \gets \Vpil{{\z}_{<l}, \x}$; $m_q\gets\mupl{{\z}_{<l}, \x}$
        \Comment{Encoder}

        \State $S_p \gets \Vtil{{\z}_{<l}}$; $m_p\gets\mutl{{\z}_{<l}}$
        \Comment{Prior}
  
        \State $\zl\gets \left(S_q + \lambda_lS_p\right)^{-1} 
        \left(S_qm_q + \lambda_lS_pm_p\right)$       
        \EndFor
            \State \textbf{return} $\enc{\x}=(\z_{0},\z_{1}, \cdots, \z_{L-1})$
            \end{algorithmic}
\end{algorithm}
    
In algorithm~\ref{algo:hierarchical_latent_reg}, the latent variables $\zl$ are determined in a hierarchical fashion starting from the coarsest to the finest one. 
As defined in relations~\eqref{eq:gaussian_hvae}, the conditionals $\qp{\zl|\x, \z_{<l}}$ and $\pt{\zl|\z_{<l}}$ are Gaussian. Therefore, the minimization at each step can be viewed as an interpolation between the mean of the encoder $\qp{\zl|\x, \z_{<l}}$ and the prior $\pt{\zl|\z_{<l}}$, given some weights conditioned by the covariance matrices and the temperature $\tau_l$.
The solution of each minimization problem in $\zl$ is given by:
\begin{align}
    \hat{\zl} \hspace{-1pt}= &\hspace{-1pt}\left(\Vpil{\hat{\z}_{<l}, \x} + \lambda_l\Vtil{\hat{\z}_{<l}}\right)^{-1} \\
    &\hspace{-1pt}\left(\Vpil{\hat{\z}_{<l}, \x}\mupl{\hat{\z}_{<l}, \x} \hspace{-2pt}+ \hspace{-2pt}\lambda_l\Vtil{\hat{\z}_{<l}}\mutl{\hat{\z}_{<l}}\hspace{-1pt}\right)\nonumber
\end{align}
where $\lambda_l=1/{\tau_l^2}-1$. 
In the following, we denote as $\hat{\z} := \enc{\x}$ the output of the hierarchical encoding of Algorithm~\ref{algo:hierarchical_latent_reg}.
In appendix~\ref{sec:algo1-global-min} we show that this algorithm finds the global optimum of $J_2(\x,\cdot)$ under mild assumptions.

The final PnP-HVAE procedure to solve an inverse problem with the HVAE prior is presented in Algorithm~\ref{algo:final}.
\begin{algorithm}[!h]
    \caption{PnP-HVAE - Restoration  by solving~\eqref{eq:J1}}\label{algo:final}
    \begin{algorithmic}\small
        \State $k \gets 0$; $res \gets + \infty$; 
        initialize $\x^{(0)}$
        \While{$res > tol$}
\State \textcolor{gray}{\% $\min_{\z} J_2(\x^{(k)},\z)$}
	\Comment{Optimize~\eqref{eq:J2} w.r.t. $\z$ using Alg. \ref{algo:hierarchical_latent_reg}}
	\State $\z^{(k+1)} = E_{\tauvec} (\x^{(k)})$

        \State \textcolor{gray}{\% $\min_{\x} J_1(\x,\z^{(k+1)})$}
        \Comment{Optimize~\eqref{eq:J1} w.r.t. $\x$}
        \State $\x^{(k+1)} = \left(A^tA+\frac{\sigma^2}{\gamma^2}\id\right)^{-1}
        \hspace{-4pt}\left(
            A^t\y+\frac{\sigma^2}{\gamma^2}\mut{\z^{(k+1)}}
        \right)$ 
        \State $res \gets ||\x^{(k+1)}-\x^{(k)}||$;  $k \gets k+1$
        \EndWhile
        \State \Return $\x^{(k)}$
    \end{algorithmic}
\end{algorithm}

\section{Convergence analysis}\label{sec:convergence}
We now analyse the convergence of Algorithm~\ref{algo:final}. Following the work of~\cite{attouch2010proximal},  the alternate optimization scheme converges if  $\qp{\z|x}=\pt{\z|\x}$ and the greedy optimization scheme in Algorithm~\ref{algo:hierarchical_latent_reg} actually solves $\min_{\z} %
J_2(\x, \z)$. %
In practice, it is difficult to verify if these hypotheses hold. 
We propose to theoretically study algorithm~\ref{algo:final}, and next verify empirically that the assumptions are met.

In section~\ref{ssec:pnp}, we reformulate Algorithm~\ref{algo:final} as a Plug-and-Play algorithm, where the HVAE reconstruction takes the role of the denoiser. 
Then we study in section~\ref{ssec:fp_conv} the fixed-point convergence of the algorithm. %
Finally, section~\ref{ssec:num_conv} contains numerical experiments with the patch HVAE architecture proposed in section~\ref{ssec:patchVDVAE}. We empirically show that the patch architecture satisfies the aforementioned technical  assumptions and then illustrate the numerical  convergence and the stability of our alternate algorithm.

\subsection{Plug-and-Play HVAE}\label{ssec:pnp}
In this section we  make the assumption that the HVAE decoder is Gaussian with a constant variance on its diagonal~\eqref{eq:decoder_constant}.
We rely on the proximal operator of a convex function~$f$ that is defined as $\prox_f(\x) = \arg\min_{\bm{u}} f(\bm{u}) + \frac{1}{2}||\x - \bm{u}||^2$.

\begin{proposition}
    \label{prop:pnp_vae}
    Assume the decoder is defined as in~\eqref{eq:decoder_constant}.
    Denote $\HVAE(\x, \tauvec) := \mut{\enc{\x}}$. Then the alternate scheme described in Algorithm~\ref{algo:final} writes
    \begin{equation}
        \label{eq:pnp-vae}
        \x_{k+1} = \prox_{\gamma^2f}\left(\HVAE\left(\x_k, \tauvec\right)\right).
    \end{equation}
\end{proposition}
From relation~\eqref{eq:pnp-vae}, algorithm~\ref{algo:final} is a Plug-and-Play Half-Quadratic Splitting method~\cite{ryu2019plug} where the role of the denoiser is played by the reconstruction $\HVAE\left(\x_k, \tauvec\right)$.
We now derive from relation~\eqref{eq:pnp-vae} sufficient conditions to establish the convergence of the iterations. %
\subsection{Fixed-point convergence}\label{ssec:fp_conv}
Let us denote $\T$ the operator corresponding to one iteration of \eqref{eq:pnp-vae}:
   $ \T(\x) = \prox_{\gamma^2}f\left(\HVAE\left(\x, \tauvec\right)\right)$.
The Lipschitz constant of $\T$ can then be expressed as a function of $f$ and the HVAE reconstruction operator $\HVAE\left(\x_k, \tauvec\right)$.
\begin{proposition}[Proof in supplementary]
    \label{eq:prop_lipschitz}
   Assume that the decoder has a constant variance $\Vti{\z} = \frac{1}{\gamma^2}\id$ for all $\z$; and  the autoencoder with latent regularization is $L_{\tauvec}$-Lipschitz, {\em i.e.} $\forall \bf{u}$, $\bf{v} \in \R^n$: $
            || \HVAE\left(\bf{u}, \tauvec\right) - \HVAE\left(\bf{v}, \tauvec\right) || \leq L_{\tauvec} ||\bf{u} -\bf{v}||$.
   Then, denoting as $\lambda_{\min}$ the smallest eigenvalue of $A^tA$, we have
    \begin{equation}\label{eq:lips_const}
        ||\T({\bf u}) - \T({\bf v})|| \leq \frac{\sigma^2}{\gamma^2\lambda_{\min} + \sigma^2}L_{\tau}||{\bf u} -{\bf v}||.
    \end{equation}
\end{proposition}

\begin{corollary}\label{cor:convergence}
If $\HVAE\left(\x_k, \tauvec\right)$  is $L_{\tauvec}<1$-Lipschitz, then iterations~\eqref{eq:pnp-vae} converge.
\end{corollary}

\begin{proof}
If  $L_{\tauvec} < 1$, then $\HVAE\left(\x_k, \tauvec\right)$ is a contraction. Hence $T$ is also a contraction form proposition~\ref{eq:prop_lipschitz} and consequently, Banach theorem ensures the convergence of the iteration $\x_{k+1} = T(\x_k)$  to a fixed point of $T$.
 \end{proof}

\begin{proposition}[Proof in supplementary]
    \label{prop:fixed_point}
    $\x^{\star}$ is a fixed point of $\T$ if and only if:
    \begin{equation}
        \label{eq:fixed_point}
        \nabla f(\x^{\star}) =  \frac{1}{\gamma^2}\left(\HVAE\left(\x^{\star}, \tauvec\right)-\x^{\star}\right)
    \end{equation}
\end{proposition}

Proposition~\ref{prop:fixed_point} characterizes the solution of the latent-regularization scheme, in the case where the HVAE reconstruction is a contraction.
Under mild assumptions, the fixed point condition can be stated as a critical point condition
$$\nabla f(\x^*)+\nabla g(\x^*)=0,$$
of the objective function $f(\x) + g(\x) = - \log p(\y|\x) - \log \ptt{\x}$, where  the tempered prior is  the marginal
$ \ptt{\x} := \int \ptt{\x,\z} d\z $
of the joint tempered prior defined in~\eqref{eq:hvae_joint_model_temp}.
This result, detailed in the supplementary, follows from an interpretation of $\HVAE(\x,\tauvec)$ as an MMSE denoiser. As a consequence Tweedie's formula provides the link between the right-hand side of equation~\eqref{eq:fixed_point} and $\nabla g$.

\subsection{Numerical convergence with PatchVDVAE }\label{ssec:num_conv}
We illustrate the numerical convergence of Algorithm~\ref{algo:final}. We first analyse the Lipschitz constant of the HVAE reconstruction with the PatchVDVAE architecture proposed  in section~\ref{ssec:patchVDVAE}. Then we study the empirical convergence of the algorithm and show that it outperforms the baseline optimisation of the joint MAP~\eqref{eq:J1} with the Adam optimizer.

\noindent
{\bf   Lipschitz constant of the HVAE reconstruction.}
Corollary~\ref{cor:convergence} establishes the fixed point convergence of our proposed optimization algorithm under the hypothesis that the reconstruction with latent regularization is a contraction, \emph{i.e.} $L_{\tauvec} < 1$. 
We now show thanks to an empirical estimation of the Lipschitz constant $L_{\tauvec}$ that our PatchVDVAE network empirically satisfies such a property when applied to noisy images.  %
We present  in figure~\ref{fig:hist_lipschitz} the histograms of the ratios $r={||\HVAE(\bf{u}, \tauvec)-\HVAE(\bf{v}, \tauvec)||}/{||\bf{u}-\bf{v}||}$, where $\bf{u}$ and $\bf{v}$ are natural images extracted from the BSD dataset and corrupted with  white Gaussian noise. These ratios give a lower bound for the true Lipschitz constant $L_{\tauvec}$.
Although it is possible to set different temperature $\tau_l$ at each level, we fixed a constant temperature amongst all levels to limit the number of hyperparameters.
We realized tests for 3 temperatures $\tau \in \{0.6, 0.8, 0.99\}$, and 3 noise levels $\sigma \in \{0, 25, 50\}$.
On clean images ($\sigma=0$), the distribution of ratios in close to $1$. This suggests that the HVAE is well trained and accurately  models clean images. In some rare case, a ratio $r\geq 1$ is observed for clean images. This indicates that the reconstruction is not a contraction everywhere, in particular on the manifold of clean images. 

On noisy images $\sigma>0$, the reconstruction behaves as a contraction, as the ratio $r<1$ is always  observed. Moreover, reducing the temperature of the latent regularization $\tau$ increases the strength of the contraction. This suggests that with the trained PatchVDVAE architecture, the hypothesis $L_{\tauvec}<1$ in  Corollary~\ref{cor:convergence} holds for noisy images.

\begin{figure}
    \includegraphics[width=\columnwidth,trim={0 0cm 0 .4cm},clip]{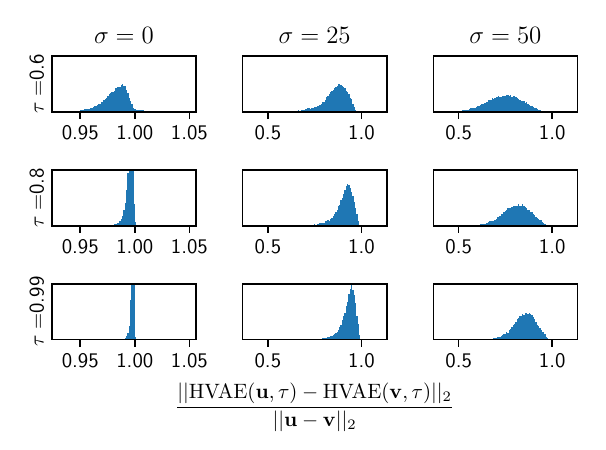}\vspace{-0.2cm}
    \caption{\label{fig:hist_lipschitz}Numerical estimation of the Lipschitz constant of PatchVDVAE reconstruction  with different temperatures $\tau$. We present the histogram of ratio values $\frac{||\HVAE(\bf{u}, \tau)-\HVAE(\bf{v}, \tau)||}{||\bf{u}-\bf{v}||}$, where $\bf{u}$ and $\bf{v}$ are  natural images corrupted with white Gaussian noise of different standard deviations $\sigma$. For noisy images ($\sigma>0)$, the observed Lipschitz constant is always less than $1$.\vspace{-0.2cm}}
\end{figure}

\noindent
{\bf   Convergence of Algorithm~\ref{algo:final}}
We now illustrate the effectiveness of PnP-HVAE through comparisons with the optimization of the objective $J_1(\x_k,\z_k)$ in~\eqref{eq:J1} using  the Adam algorithm~\cite{kingma2014adam} for two  learning rates $lr \in \{0.01, 0.001\}$.  The left plot in figure~\ref{fig:conv} shows that Adam is able to estimate a better minimum of $J_1$. However, our alternate algorithm requires a smaller number of iterations to converge. %

On the other hand, as illustrated by the right plot in figure~\ref{fig:conv}, the use of Adam involves numerical instabilities. Oscillations of the ratio $ L_k :=\frac{||\T\left(\x_{k+1}\right) - \T\left(\x_{k}\right)||}{||\x_{k+1} - \x_{k}||} %
$ are even increased  with larger learning rates, whereas our method provides a stable  sequence of iterates.

More important, we finally exhibit the better quality of the restorations obtained with our alternate algorithm  %
on inpaiting, deblurring and super-resolution of face images. In these experiments, we used the hierarchical VDVAE model~\cite{child2021very}  trained on the FFHQ dataset \cite{karras2019style}.
Figure~\ref{fig:face_restoration} (see $2$nd and $4$th columns) and table~\ref{table:comp_ILO} (PSNR, SSIM and LPIPS scores) illustrate that the quality of the images restored with our alternate optimization algorithm is higher than the ones obtained with Adam. This suggests that for image restoration purposes, our optimization method is able to find a more relevant local minimum of $J_1$ than Adam.

\begin{figure}[t]
    \label{fig:convergence}
    \centering
    \begin{tabular}{cc}
    \hspace{-.2cm}\includegraphics[height=.33\linewidth]{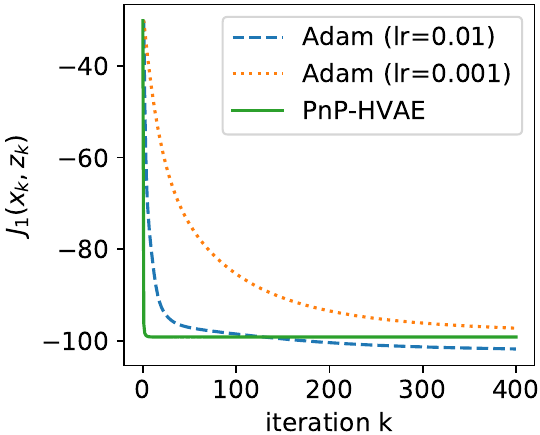}&\hspace{-.2cm}\includegraphics[height=.33\linewidth]{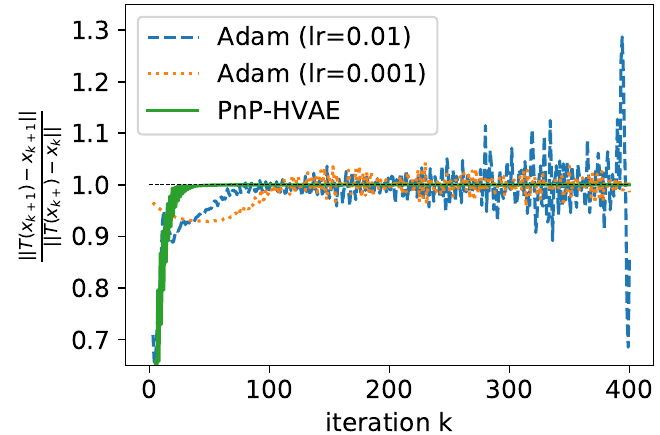}\vspace{-0.4cm}
    \end{tabular}
    \caption{\label{fig:conv} Comparison of the convergence of PnP-HVAE algorithm~\ref{algo:final} with respect to the baseline Adam optimizer, on a deblurring problem.
     Left (Convergence of the function value): PnP-HVAE converges faster to a minimum of the joint posterior  $J_1(\x_k,\z_k)$  in~\eqref{eq:J1}. 
   Right (Convergence of iterates $\x_k$): PnP-HVAE is more stable than Adam.\vspace{-0.1cm}}
\end{figure}

\begin{figure}[ht]
    \center
    \includegraphics[width=0.95\columnwidth]{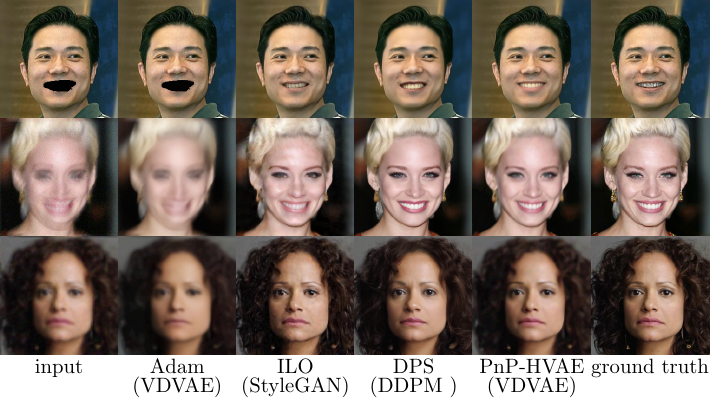}\vspace{-0.1cm}
    \caption{\label{fig:face_restoration}
        Visual comparaison of image restoration methods based on deep generative models. We studied $3$  tasks on face images:  inpainting (top), deblurring (middle), super-resolution (bottom). 
        Contrary to the optimization of the objective~\eqref{eq:J1} with Adam, our alternate algorithm generates realistic results, on par with ILO~\cite{daras2021intermediate}, while remaining consistant with the observation.\vspace{-0.3cm}
    }
\end{figure}

\newcolumntype{x}[1]{>{\centering\arraybackslash}p{#1}}
\begin{table}[ht]\small
    \centering
        \begin{tabular}{x{1.15cm}x{1.5cm}x{.65cm}x{.65cm}x{.65cm}x{1cm}}
           & &PSNR$\uparrow$ & SSIM$\uparrow$ & LPIPS$\downarrow$ & time (s) \\
        \hline    
        \multirow{2}{*}{SR $\times 4$}&Adam & $28.56$ & $0.75$ & $0.38$ & $\underline{26}$  \\
       \multirow{2}{*}{$\sigma=3$}& ILO & $\underline{28.80}$ & $\underline{0.78}$ & $\mathbf{0.17}$ & $34$\\
       & PnP-HVAE & $\mathbf{29.32}$ & $\mathbf{0.82}$ & $0.28$ & $\mathbf{15}$\\ 
       & DPS & $27.53$ & $0.76$ & $\underline{0.21}$ & $153$\\
        \hline 
        Deblurring &Adam & $26.69$ & $0.75$ & $0.27$ &$ \underline{12}$  \\
        (motion) &ILO & $\underline{29.01}$ & $\underline{0.80}$ & $\underline{0.20}$ & $34$ \\
        $\sigma=8$&PnP-HVAE & $\mathbf{30.40}$ & $\mathbf{0.84}$ & $\mathbf{0.16}$ & $\mathbf{10}$\\
         & DPS & $28.70$ & $\underline{0.80}$ & $0.23$& $142$\\
        \hline 
         Deblurring &Adam & $\underline{30.17}$ & $\underline{0.83}$ & $\underline{0.21}$& $\underline{12}$\\
         (Gaussian) &ILO & $29.12$ & $0.79$ & $\mathbf{0.17}$ & $34$\\
        $\sigma=8$&PnP-HVAE& $\mathbf{30.81}$ & $\mathbf{0.86}$ & $0.24$ & $\mathbf{10}$\\
        & DPS & $29.14$ & $0.81$ & $ $0.24$ $ & $142$\\
        \end{tabular}\vspace*{-0.1cm}
    \caption{\label{table:comp_ILO}Quantitative evaluation on face restoration. Best results in {\bf bold}, second best \underline{underlined}.\vspace{-0.2cm}}
\end{table}

\section{Image restoration results}\label{sec:expe}
We present in section~\ref{ssec:faces} an application of PnP-HVAE on face images, using a pretrained state-of-the-art hierarchical VAE. %
Next, we study the application of our framework to natural images. To that end, we introduce  in section~\ref{ssec:patchVDVAE}  a patch hierachical VAE architecture, that is able to model natural images of different resolutions. In section~\ref{ssec:app_nat}, we provide deblurring, super-resolution and inpainting experiments to demonstrate the relevance of the proposed method.

\subsection{Face Image restoration (FFHQ)}\label{ssec:faces}
We first demonstrate the effectiveness of PnP-HVAE on highly structured data, by performing face image restoration.
Latent variable generative models can accurately model structured images such as face images \cite{karras2019style,vahdat2020nvae,child2021very,kingma2018glow}, and then be used to produce high quality restoration of such data.
In our experiments, we use the VDVAE model of~\cite{child2021very}, pre-trained on the FFHQ dataset~\cite{karras2019style}, as our hierarchical VAE prior.
VDVAE has $L=66$ latent variable groups in its hierarchy and generates images at resolution $256\times256$.

We compare PnP-HVAE with two restoration methods based on different class of generative models, namely the intermediate layer optimization algorithm (ILO)~\cite{daras2021intermediate} and the diffusion posterior sampling method (DPS)~\cite{chung2022diffusion}. ILO is a GAN inversion method which optimizes the image latent code along with the intermediate layer representation of a StyleGAN2 generative network~\cite{karras2020analyzing} to generate an image consistent with a degraded observation.
DPS use denoising diffusion probabilistic model~\cite{song2020score,ho2020denoising} as a prior, and produce a sample from the posterior by conditioning each iteration of the sampling process on $\y$.
We use the official implementation of ILO, along with a StyleGAN2 model that was trained for 550k iterations on images of resolution $256\times256$ from FFHQ~\cite{stylegan2pytorch}. 
For DPS, we use the official implementation as well.
As VDVAE and StyleGAN models are not trained on the same train-test split of FFHQ, we chose to evaluate the methods on a subset of 100 images from the CelebA dataset~\cite{liu2018large}. %
For super-resolution, the degradation model corresponds to the application of a Gaussian low-pass filter followed by a $\times 4$ sub-sampling, and the addition of a Gaussian white noise with $\sigma=3$.
For the deblurring, we considered motion blur and  Gaussian kernels, both with a noise level $\sigma=8$. %

We provide quantitative comparisons in table~\ref{table:comp_ILO}, along with a visual comparison of the results in figure~\ref{fig:face_restoration}.
PnP-HVAE has the best  PSNR and SSIM results for all the considered restoration tasks, while ILO provides better results  for the perceptual distance.
By jointly optimizing the image and its latent variable, PnP-HVAE provides  results that are both realistic and consistent with the degraded observation.
On the other hand,  ILO  only optimizes on an extended latent space. This method generates  sharp and realistic images with better LPIPS scores,  
but the results lack  of consistency with respect to the observation, which explains the overall lower PSNR performance. 
DPS produces highly realistic samples, however DPS is limited by its long inference time, as it requires one network function evaluation and one backpropagation operation through the network at each of the 1000 sampling steps required to generate one image.

\subsection{PatchVDVAE: a HVAE for natural images}\label{ssec:patchVDVAE}
Available generative models in the literature operate on images of  fixed resolutions and %
are either restrained to datasets of limited diversity, or even to registered face images~\cite{kingma2018glow,child2021very, vahdat2020nvae, karras2019style}, or requiring additional class information~\cite{brock2018large, dhariwal2021diffusion, song2020score, luhman2022optimizing}.
Fitting an unconditional model on natural images appears to be a more difficult task, as their resolution can change, and their content is highly diverse.
The complexity of the problem can be reduced by learning a prior model on patches of reduced dimension. 
For image restoration problems, the patch model can be reused on images of higher dimensions~\cite{zoran2011learning,prost2021learning,altekruger2023patchnr}. When the model is a full CNN, the prior on the set of the  patches can  be computed efficiently by applying the network on the full image~\cite{prost2021learning}.

We thus introduce  patchVDVAE, a fully convolutional hierarchical VAE.
Contrary to existing HVAE models whose resolution is constrained by the constant tensor at the input of the top-down block, patchVDVAE can generate images of different resolutions by controlling the dimension of the input latent.
This amounts to defining a prior on patches whose dimension corresponds to the receptive field of the VAE. A similar model is used for image denoising in~\cite{prakash2021interpretable}.

For PatchVDVAE architecture, we use the same bottom-up and top-down blocks as VDVAE~\cite{child2021very}, and replace the constant trainable input in the first top-down block by a latent variable, to make the model fully convolutional (details on the  architecture are given in the supplementary). 
The training dataset is composed of $128\times 128$ patches extracted from a combination of DIV2K~\cite{agustsson2017ntire} and Flickr2K~\cite{Lim_2017_CVPR_workshops} datasets.
We perform data augmentation by extracting  patches at $3$ resolutions: HR-images and $\times 2$ and $\times 4$ downscaled images. 
The model is trained for $7.10^5$ iterations with a batch size of $64$. Following the recommendation of~\cite{hazami2022efficient}, we use Adamax optimizer with an exponential moving average and gradient smoothing of the variance.
We set the decoder model to be a Gaussian with diagonal covariance, as in~\cite{luhman2022optimizing}.
PatchVDVAE is fully convolutional and can generate images of dimension that are multiples of $64$ as illustrated by
figure~\ref{fig:vdvae}.

\newlength{\patchwidth}
\setlength{\patchwidth}{0.135\columnwidth}
\begin{figure}[!ht]
    \centering
    \begin{subfigure}[t]{.34\columnwidth}\hspace{0.1cm}
        \setlength{\tabcolsep}{0.02pt}
\renewcommand{\arraystretch}{0}
        \begin{tabular}{*{2}{p{1.03\patchwidth}}}%
            \includegraphics[width=\patchwidth]{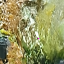} &
            \includegraphics[width=\patchwidth]{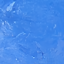} \\
            \includegraphics[width=\patchwidth]{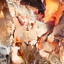} &
            \includegraphics[width=\patchwidth]{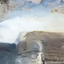} %
        \end{tabular}
    \end{subfigure}\hspace{-0.15cm}
    \begin{subfigure}[t]{.64\columnwidth}
\begin{tabular}{cc}\vspace{-0.1cm}
\includegraphics[width=2\patchwidth]{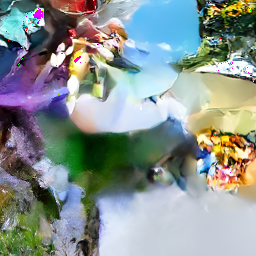}&
        \includegraphics[width=2\patchwidth]{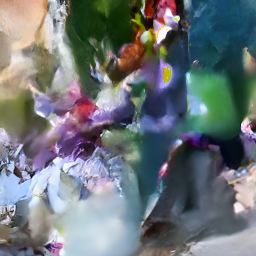}\end{tabular}
    \end{subfigure}
    \caption{\label{fig:vdvae} Left: $64\times64$ patches samples from our patchVDVAE model trained on patches from natural images.
    Right: PatchVDVAE is fully convolutional and it can generate images of higher resolution (here: $128\times128$).\vspace{-0.2cm}}
\end{figure}

\subsection{Natural images restoration}\label{ssec:app_nat}
We  evaluate PnP-HVAE on natural image restoration.
For each task, we report the average value of the PSNR, the SSIM, and the LPIPS metrics on $20$ images from the test set of the BSD dataset~\cite{MartinFTM01}.

\noindent
{\bf Image deblurring.}
In the experiments, we consider $2$ Gaussian kernels and $2$ motion blur kernels from~\cite{levin2009understanding}, with $3$ different noise levels 
$\sigma \in \{2.55, 7.65, 12.75\}$.
As a baseline we consider  EPLL~\cite{zoran2011learning}, which learns a prior on image patches with a Gaussian mixture model.
We also compare PnP-HVAE  with PnP-MMO and GS-PnP, $2$ competing convergent Plug-and-Play methods based on CNN denoisers.
PnP-MMO~\cite{pesquet2021learning} restricts the denoiser to be contraction in order to guarantee the convergence of the PnP forward-backard algorithm. GS-PnP~\cite{hurault2022gradient} considers a gradient step denoiser and reaches state-of-the-art performances of non converging methods~\cite{zhang2021plug}.
We set the temperature $\tau$  in our method as $0.95$, $0.8$ and $0.6$ for noise levels $2.55$, $7.65$ and $12.75$ respectively, and we let it run for a maximum of $50$ iterations. 
For the three compared methods we use the official implementations and pre-trained models provided by the respective authors. %
Details on the choice of hyperparameters for the concurrent methods are provided in the supplementary material.
Figure~\ref{fig:deblurring_bsd} illustrates that our method provides correct deblurring results. 
According to table~\ref{tab:deb}, the performance of PnP-HVAE is between those of EPLL and GS-PnP and it outperforms PnP-MMO for large noise levels.

\begin{table}
\begin{center}\footnotesize

    \begin{tabular}{>{\centering}m{.3cm}*{5}{c}}
    $\sigma$ &Method & PSNR$\uparrow$ & SSIM$\uparrow$ & LPIPS$\downarrow$  \\ 
    \hline
    \multirow{4}{*}{\vcell{$2.55$}}
    & PnP-HVAE & $27.75$ & $0.79$ & $0.31$\\
    & GS-PNP \cite{hurault2022gradient} & $\mathbf{29.59}$ & $\mathbf{0.84}$ & $\mathbf{0.22}$\\
    & EPLL \cite{zoran2011learning} & $26.49$ & $0.71$ & $0.36$\\ 
    & PnP-MMO \cite{pesquet2021learning} & $\underline{29.50}$ & $\underline{0.83}$ & $\underline{0.20}$ \\ \hline
    \multirow{4}{*}{\vcell{$7.65$}}
    & PnP-HVAE & $\underline{26.36}$ & $\underline{0.72}$ & $\underline{0.40}$\\
    & GS-PNP \cite{hurault2022gradient} & $\mathbf{27.33}$ & $\mathbf{0.77}$ & $\mathbf{0.31}$\\
    & EPLL \cite{zoran2011learning} & $24.04$ & $0.66$ & $0.45$ \\ 
    & PnP-MMO \cite{pesquet2021learning} & $25.34$ & $0.69$ & $0.34$\\
    \hline
    \multirow{4}{*}{\vcell{$12.75$}}
    & PnP-HVAE & $\underline{25.12}$ & $\mathbf{0.73}$ & $\underline{0.47}$\\
    & GS-PNP \cite{hurault2022gradient} & $\mathbf{26.32}$ & $\mathbf{0.73}$ & $\mathbf{0.37}$\\
    & EPLL \cite{zoran2011learning} & $23.28$ & $0.61$ & $0.51$ \\ 
    & PnP-MMO \cite{pesquet2021learning} & $22.42$ & $0.53$& $0.54$ \\
    \hline
    &\vspace*{-.3cm}\\
            \multicolumn{2}{c}{Blur and motion kernels}& \multicolumn{3}{c}{
        \includegraphics*[scale=1]{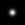}\;\includegraphics*[scale=1]{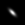}\;\includegraphics*[scale=1]{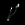}\;\includegraphics*[scale=1]{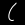}} 
    \end{tabular}
        \caption{\label{tab:deb}Comparison  of PnP-HVAE  and other restoration methods on deblurring. Results are averaged on $4$ kernels.\vspace{-0.2cm}}%
    \end{center}
\end{table}

\begin{figure}
    \begin{subfigure}[h]{\linewidth}
        \centering
        \includegraphics*[width=\columnwidth]{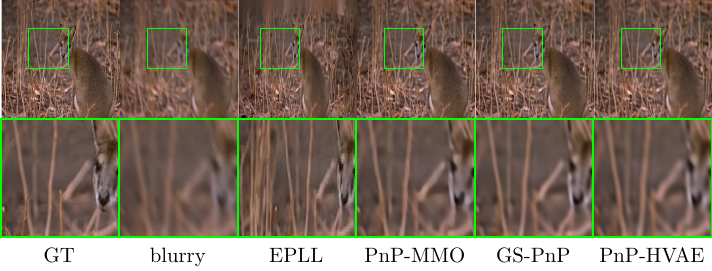}\vspace{-0.1cm}
        \caption{Gaussian blur, $\sigma=2.55$}
    \end{subfigure}
    \begin{subfigure}[h]{\linewidth}
        \centering
        \includegraphics*[width=\columnwidth]{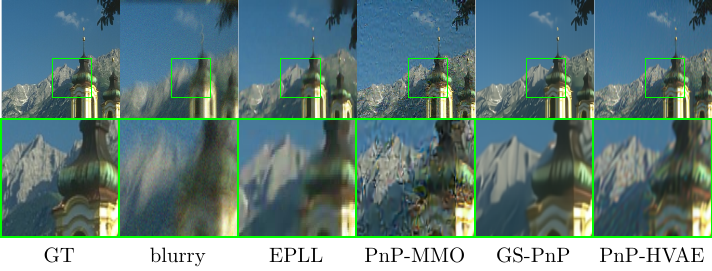}\vspace{-0.1cm}
        \caption{Motion blur, $\sigma=7.65$}
    \end{subfigure}\vspace*{-0.1cm}
    \caption{\label{fig:deblurring_bsd} Natural image deblurring\vspace{-0.1cm}}
\end{figure}

\noindent
{\bf Image inpainting.}
Next we consider the task of noisy image inpainting. 
We compose a test-set of 10 images from the validation set of BSD~\cite{MartinFTM01} and we create masks
  by occluding diverse objects of small size in the images. 
A Gaussian white noise with $\sigma=3$ is added to the images.
As a comparaison, we still consider GS-PnP and EPLL.
For PnP-HVAE, the temperature is set to $\tau=0.6$, and the algorithm is run for a maximum of $200$ iterations, unless the residual $||\x_{k+1}-\x_k||$ is on a plateau.
We provide on Table~\ref{tab:inpainting_bsd} the distortion metrics with the ground truth, as well as a visual
\begin{table}

\begin{center}
    \begin{tabular}{cccc}
        & PSNR$\uparrow$ & SSIM$\uparrow$ &LPIPS$\downarrow$ \\\hline
        PnP-HVAE  & $\mathbf{29.54}$ & $\mathbf{0.93}$ & $\mathbf{0.06}$\\
        GS-PNP & $28.52$ & $\mathbf{0.93}$ & $0.09$\\
        EPLL & $\underline{29.16}$ & $\mathbf{0.93}$ & $\mathbf{0.06}$\\
    \end{tabular}
    \caption{\label{tab:inpainting_bsd}Quantitative evaluation for inpainting on BSD.}
    \end{center}
\end{table}
comparison on figure~\ref{fig:inpainting_bsd}. 
With its hierarchical structure,  PnP-HVAE outperforms the compared methods. \vspace{0.05cm}
\begin{figure}[!h]
    \includegraphics[width=\columnwidth]{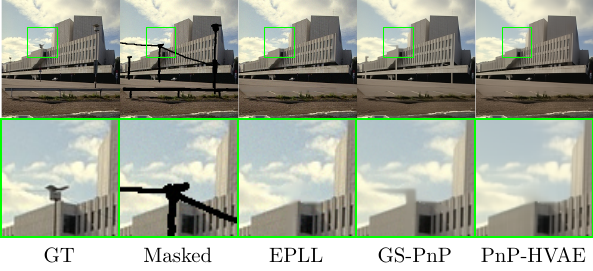}\vspace{-0.1cm}
    \caption{\label{fig:inpainting_bsd}Natural image inpainting\vspace{-0.3cm}}
\end{figure}

\noindent {\bf Effect of the temperature.}
PnP-HVAE gives control on the temperature of the prior over the latent space.
In figure~\ref{fig:temp_effect}, we illustrate that reducing the temperature increases the strength of the regularization prior. In this example the tuning $\tau=0.7$ produces the best performance.
\begin{figure}[!ht]
   
    \includegraphics[width=\columnwidth]{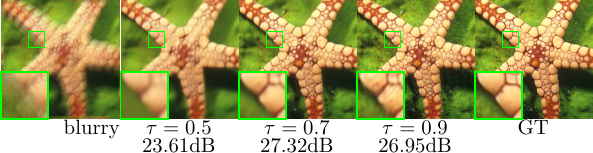}\vspace{-0.15cm}
    \caption{ \label{fig:temp_effect} Effect of the temperature in PnP-VAE on a deblurring problem, with $\sigma=7.65$.\vspace{-0.15cm}}
\end{figure}

\noindent {\bf Effect of the number of latent groups}
We study the effect of the number of latent groups $L$ on the hierarchical model on the restoration performance.
It has been observed that HVAEs outperform non-hierarchical VAEs in terms of likelihood score~\cite{sonderby2016ladder}, and that increasing the number of latent groups in the hierarchy improves the modelling performance of HVAE for a fixed number of parameters~\cite{child2021very}.
Therefore we can expect that the gain in modelling performance due to a higher $L$ will translate into a gain in restoration performance using our method.
We train different patchVDVAE models, with different number of latent groups $L$. In order to keep the number of trainable parameters constant, we replace stochastic top-down blocks with deterministic blocks in our network with the higher $L$ value ($L=36$). We evaluate the different models on image deblurring, using the same experimental settings as the one described in subsection~\ref{ssec:app_nat}.
The results in table~\ref{tab:ablation_L} show that increasing the number  of stochastic groups (L) has a positive  effect on the validation ELBO, up to $L=18$, and that a better ELBO correlates with a better restoration performance.

\begin{table}
    \centering
    \caption{Effect of the number $L$ of latent groups on the restoration performance, measured in PSNR (dB), for image deblurring. We observed similar trends for the LPIPS and SSIM metrics.
    \label{tab:ablation_L}}
      \begin{tabular}[h]{c|cccc} 
         & $L=6$ & $L=12$ & $L=18$ & $L=36$ \\\hline
      $\sigma=2.55$ & $27.25$ &$\mathbf{27.87}$ & $\underline{27.82}$ & $27.71$\\
      $\sigma=7.65$ & $26.10$ &$26.41$ & $\mathbf{26.74}$ & \underline{$26.51$}\\
      $\sigma=12.75$ & $24.78$ &$25.16$ & $\mathbf{25.57}$ & \underline{$25.27$}\\\hline
      ELBO$\uparrow$(val) & $-1.24$ &$-1.14$ & $\mathbf{-1.10}$ & $\mathbf{-1.10}$ \vspace{-0.35cm}
   \end{tabular}
\end{table}

\section{Conclusion}
We proposed PnP-HVAE, a method using hierarchical variational autoencoders as a prior to solve image inverse problems.
Motivated by an alternate optimization scheme, PnP-HVAE exploits the encoder of the HVAE to avoid backpropagating through the generative network. 
We derived sufficient conditions on the HVAE model to guarantee the convergence of the algorithm. We have verified empirically that PnP-HVAE satisfies those conditions.
By jointly optimizing over the image and the latent space, PnP-HVAE  produces realistic results that are more consistent with the observation than GAN inversion  on a specialized dataset.
PnP-HVAE can also restore natural images of any size  using our PatchVDVAE model trained on natural images patches.

\ificcvfinal
\section*{Acknowledgements}
This study has  been carried out with financial support from the French Research Agency through the PostProdLEAP project (ANR-19-CE23-0027-01).
Experiments presented in this paper were carried out using the PlaFRIM experimental testbed, supported by Inria, CNRS (LABRI and IMB), Université de Bordeaux, Bordeaux INP and Conseil Régional d’Aquitaine (see https://www.plafrim.fr)
\fi
\clearpage

{\small
\bibliographystyle{ieee_fullname}
\bibliography{biblio}
}

\clearpage

\onecolumn
\appendix
\section*{Summary}

This supplementary material contains:
\begin{itemize}
\item proofs of the theoretical results of the main paper in section~\ref{proofs}
\item additional implementation details in section~\ref{details} 
\item a discussion on the contractivity of the autoencoder and its fixed points in section~\ref{sec:contractivity}
\item additional  comparisons with the competing methods  in section~\ref{add}
\end{itemize}
\section{Proofs of the main results}\label{proofs}
In this section we provide proofs relative to Algorithm~\ref{algo:hierarchical_latent_reg}, Proposition~\ref{eq:prop_lipschitz}, Proposition~\ref{prop:fixed_point} and the characterization of the fixed point given by  Algorithm~\ref{algo:final}.

\subsection{Global minimum of the hierarchical Gaussian negative log-likelihood}\label{sec:algo1-global-min}

In this section we show that under certain conditions Algorithm~\ref{algo:hierarchical_latent_reg} actually computes the global minimum of $J_2(\x,\z)$ with respct to $\z$.
To reach that conclusion we first decompose the objective function into several terms (equation~\eqref{eq:J2decomp} in proposition~\ref{prop:J2decomp}). Since many of these terms do not depend on $\z$ we conclude that
$$ \arg\min_{\z} J_2(\x,\z) = \arg\min_{\z} A(\z) + B(\z).$$

Furthermore, since the second term ($B(\z)$) only depends on $\z$ via the determinant of the encoder and decoder covariances, we have that under assumption~\ref{hypo:vpcov}
$$ \arg\min_{\z} J_2(\x,\z) = \arg\min_{\z} A(\z).$$

Finally, proposition~\ref{prop:global-min} shows that a functional of the form $A(\z)$ reaches its global minimum exactly at the point $E_{\tauvec}(\x)$ computed by Algorithm~\ref{algo:hierarchical_latent_reg}. Hence under assumption~\ref{hypo:vpcov} we have that
$$ \arg\min_{\z} J_2(\x,\z) = \arg\min_{\z} A(\z) = E_{\tauvec}(\x).$$

\begin{proposition}\label{prop:J2decomp}
	The objective $J_2(\x,\z)$ in equation~\eqref{eq:J2} can be decomposed as
	\begin{align}\label{eq:J2decomp}
		J_2(\x,\z) = f(\x)-\log\pd{\x} + A(\z) + B(\z) + C
	\end{align}
	where
	\begin{align}
		A(\z) &:= \sum_{l=0}^{L-1} A_l(\zl,\z_{<l}) \label{eq:A}\\
		B(\z) &:= \sum_{l=0}^{L-1} B_l(\z_{<l}) \\
		C      &:=  \sum_{l=0}^{L-1} C_l 
	\end{align}
	and
	\begin{align}
		A_l(\zl,\zll) &:= \| \zl - m_l(\zll)\|^2_{S_l^{-1}(\zll)} \label{eq:Al} \\
		B_l(\zll)      &:= \frac12 \log \det(S_l^{-1}(\zll)) + \frac12 (1-\lambda_l)\log \det(S_{p,l}^{-1}(\zll)) \\
		C_l             &:= \frac{d_l}{2} (\log \lambda_l -\lambda_l\log(2\pi))
	\end{align}
    and
    \begin{align}
    	m_{p,l}(\zll) & := \mu_{\theta,l}(\zll)           &   S_{p,l}(\zll) &:= \Sigma^{-1}_{\theta,l}(\zll) \label{eq:mplSpl} \\
    	m_{q,l}(\zll) &:=  \mu_{\phi,l}(\zll)               &   S_{q,l}(\zll) &:= \Sigma^{-1}_{\phi,l}(\zll)  \label{eq:mqlSql} \\
    	m_l(\zll) &:= S_{q,l}(\zll) m_{q,l}(\zll) + \lambda_l S_{p,l}(\zll) m_{p,l}(\zll) &
    	S_l(\zll) &:= S_{q,l}(\zll) + \lambda_l S_{p,l}(\zll)  \label{eq:gaussianprod}  %
    \end{align}
\end{proposition}

\begin{proof}
	First observe that $\qp{\zl|\x,\zll}$ and $\pt{\zl | \zll}$ are multivariate Gaussians as stated in equation~\eqref{eq:gaussian_hvae}. Also $\pt{\zl | \zll}^{\lambda_l}$ behaves like a Gaussian with a different normalization constant, namely
	$$  \pt{\zl | \zll}^{\lambda_l} = \N(\zl;m_{p,l},\lambda_l^{-1}S^{-1}_{p,l}) D_l  $$
	where the missing normalization constant is
	$$ D_l = (2\pi)^{-\frac{d_l}{2}(1-\lambda_l)}  \lambda_l^{-\frac{d_l}{2}} \det(S_{p,l}^{-1})^{-\frac{1}{2}(1-\lambda_l)}$$
	Now $  \qp{\zl | \x, \zll} \pt{\zl | \zll}^{\lambda_l} $ is the product of two Gaussians times the correcting term $D_l$. Since the product of two Gaussians is a Gaussian we obtain
	$$  \qp{\zl | \x, \zll} \pt{\zl | \zll}^{\lambda_l}  = \N(\zl; m_l, S_l^{-1}) D_l $$
	with mean and variance given by equation \eqref{eq:gaussianprod}.
	Taking  $-\log$ in the previous expression, we get $A_l + B_l + C_l$ by grouping into $A_l$ the  terms depending on both $\zl$ and $\zll$, in $B_l$ those depending only on $\zll$, and into $C_l$ the constant terms.
\end{proof}

\begin{assumption}[Volume-preserving covariances]\label{hypo:vpcov}
	The covariance matrices of the HVAE have constant determinant (not depending on $\zll$, although this constant may depend on the hierarchy level $l$)
	\begin{align}
		\det(\Sigma_{\phi,l}(\zll,\x)) &= c_l(\x)  \\
		\det(\Sigma_{\theta,l}(\zll)) &= d_l 
	\end{align}
\end{assumption}

\begin{proposition}[Algorithm~\ref{algo:hierarchical_latent_reg} computes the global minimum of $J_2(\x,\z)$ with respect to $\z$]\label{prop:global-min}
	Under Assumption~\ref{hypo:vpcov} minimizing $J_2(\x,\z)$ w.r.t. $\z$ is equivalent to minimizing  $A(\z)$ defined in equations~\eqref{eq:A}~and~\eqref{eq:Al}, \emph{i.e.}
	$$ \arg\min_{\z} J_2(\x,\z) = \arg\min_{\z} A(\z).$$
	In addition the global minimum of $A(\z_0, \dots, \z_{L-1})$ is given by the recursion computed by Algorithm~\ref{algo:hierarchical_latent_reg}, namely:
	    \begin{equation}
		\label{eq:hierarchical_solution}
		\begin{cases}
			\z_0^{\star} &=  m_0 \\
			\z_{l}^{\star} &= m_{l}(\z_{<l}^{\star}) \quad \text{for $l \in \{1, \dots, L-1\}$}
		\end{cases}
	\end{equation}
	where $\z_{<l}^{\star} = (\z_0^{\star}, \dots, \z_{l-1}^{\star})$, and $m_l(\zll)$ as defined in equations~\eqref{eq:mplSpl}~to~\eqref{eq:gaussianprod}. Put another way, $\z^{\star} = E_{\tauvec}(\x)$ as computed by Algorithm~\ref{algo:hierarchical_latent_reg}.
\end{proposition}

\begin{proof}
According to the decomposition of $J_2$ into several terms (equation~\eqref{eq:J2decomp} in proposition~\ref{prop:J2decomp}), we observe that many of these terms do not depend on $\z$. Therefore we conclude that
$$ \arg\min_{\z} J_2(\x,\z) = \arg\min_{\z} A(\z) + B(\z).$$

Furthermore, since the second term ($B(\z)$) only depends on $\z$ via the determinant of the encoder and decoder covariances, and these determinants do not depend on $\z$ under assumption~\ref{hypo:vpcov}, we conclude the first part of the proposition, namely
$$ \arg\min_{\z} J_2(\x,\z) = \arg\min_{\z} A(\z).$$

Now let's find the global minimum of $A(\z)$.

It is clear that $A(\z_0, \dots, \z_{L-1}) \geq 0$ for all $\z_0, \dots, \z_{L-1}$. 
It is also simple to verify that:
\begin{align} \nonumber
   & A(\z_1^{\star}, \dots, \z_{L-1}^{\star})\\ =&  ||m_0 - m_0||^2_{S_1^{-1}} + \sum_{l=1}^{L-1} ||m_{l}(\z_{<j}^{\star}) - m_l(\z_{<j}^{\star})||^2_{S_l^{-1}(\z_{<l}^{\star})} \nonumber\\
    =&0.
\end{align}
Therefore the minimum value of $A$ is reached in $\z^{\star}$.
Furthermore, for any $\z \neq \z^{\star}$, let us denote by $k$ the first value in $\{0, \dots, L-1\}$ such that $\z_k \neq m_k(\z_{<k}^{\star})$. 
Then, 
\begin{align}
    A(\z_0, \dots, \z_{L-1}) \geq ||\z_k - m_k(\z_{<k}^{\star})||^2_{S_k^{-1}(\z_{<k}^{\star})} > 0,
\end{align}
which implies that $\z^{\star}$ is the unique minimum of $A$.
\end{proof}

\paragraph{Discussion on assumption 1 (volume preserving covariance)}
We showed in proposition 5 that, under assumption 1, Algorithm~\ref{algo:hierarchical_latent_reg} computes the global minimum of $J_2(\x,\z)$ with respect to $\z$.
 When optimizing $z_{l}$ in~\eqref{eq:J2decomp} we only consider the impact of $\zl$ on the distance to the Gaussian mean in $A(\z)$, while ignoring its impact on the covariance volumes in the subsequent levels in the terms $B_{l'}(\z_{<l'})$, for $l' > l$. 
 If the covariance volumes are constant as stated in assumption 1, the value of $\zl$ has no impact on the covariance volumes of the subsequent levels, and algorithm~\ref{algo:hierarchical_latent_reg} gives the global minimizer of $J_2(\x, .)$ with respect to $\z$.
In practice, the HVAE we use does not enforce the covariance matrices of $p(\zl|\z_{<l})$ and $q(\zl|\z_{<l}, \x)$ to have constant volume.
However, the experiment in figure~\ref{fig:volume-preserving-covariance} shows that the variation of $B_{l+1}(\z_{<l+1})$ is negligible in front of $A_{l}(\zl)$.
Hence, we can reasonably expect algorithm~\ref{algo:hierarchical_latent_reg} to yield the minimum of $J_2(\x,\z)$ with respect to $\z$.
For future works, we could explicitly enforce assumption 1 in the HVAE design.

\begin{figure}
    \centering
    \includegraphics*[width=0.7\columnwidth]{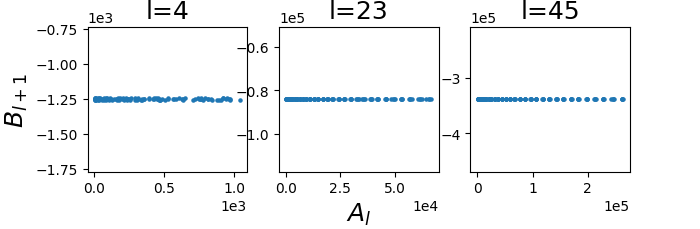}\vspace{-0.3cm}
    \caption{\small%
    Evolution of {$B_{l+1}=\log \det S_{l+1}^{-1}(z_{<l+1})$} as a function of the distance ${A_{l}} = \|z_l - \mu_l(z_{<l})||^2_{{S_l^{-1}}(z_{<l})}$. (experiment made on VDVAE).\vspace{-0.5cm}%
    }\label{fig:volume-preserving-covariance}
 \end{figure}

\subsection{Proof of Proposition~\ref{eq:prop_lipschitz} (Lipschitz constant of one iteration)}\label{app:prop_lipschitz}

\begin{proof}
    For a decoder with constant covariance $\Vti{\z} = \frac{1}{\gamma^2}\id$, we have:
    \begin{equation}
        \T(\x) = \left(\hspace{-2pt}A^tA+\frac{\sigma^2}{\gamma^2}\id\right)^{\hspace{-2pt}-1}\hspace{-3pt}
        \left(\hspace{-2pt}
            A^t\y+\frac{\sigma^2}{\gamma^2}\mut{\enc{\x}}
        \hspace{-2pt}\right)
    \end{equation}
    and then :
    \begin{equation}
        ||T({\bf u}) - T({\bf v})|| \leq \left|\left|\left(A^tA+\frac{\sigma^2}{\gamma^2}\id\right)^{-1}\right|\right|\frac{\sigma^2L_{\tau}}{\gamma^2}||{\bf u} -{\bf v}|| .
    \end{equation}
    To conclude the proof, we use that for an invertible matrix~M, $||M^{-1}||=\frac{1}{\bf{\sigma}_{min}(M)}$,  where $\bf{\sigma}_{min}(M)$ is the smallest eigenvalue of $M$.
  We also use the fact that $\alpha$ is an eigenvalue of $A^tA+\frac{\sigma^2}{\gamma^2}\id$ if and only if $\alpha=\lambda +\frac{\sigma^2}{\gamma^2}$ for an eigenvalue $\lambda\geq 0$ of the positive definite matrix $A^tA$.
\end{proof}

\subsection{Proof of Proposition~\ref{prop:fixed_point} (fixed point of PnP-HVAE)}\label{app:fixed_point}

\begin{proof}
	$\x^*$ is a fixed point of $T$ if and only if $\x^*=T(\x^*)$.
	Recalling the definition of $ \T(\x) := \prox_{\gamma^2}f\left(\HVAE\left(\x, \tauvec\right)\right)$, and the definition of proximal operator $\prox_{\gamma^2 f}(\x) = \arg\min_{\bm{t}} \gamma^2f(\bm{u}) + \frac{1}{2}||\x - \bm{t}||^2$, the fixed point condition is equivalent to
	$$ \x^* = \arg\min_{\bm{t}} \frac12 \|\bm{t}-\HVAE\left(\x^*, \tauvec\right)\|^2 + \gamma^2 f(\bm{t}).$$
	Since $f$ is convex the above condition is equivalent to
	$$ \x^*-\HVAE\left(\x^*, \tauvec\right) + \gamma^2 \nabla f(\x^*) = 0.$$
	Rearranging the terms we obtain equation~\eqref{eq:fixed_point}.
\end{proof}

Under mild assumptions the above result can be restated as follows:
$\x^*$ is a fixed point of $T$ if and only if 
$$\nabla f(\x^*)+\nabla g(\x^*)=0,$$
\emph{i.e.} whenever $\x^*$ is a \emph{critical point} of the objective function $f(\x) + g(\x) = - \log p(\y|\x) - \log \ptt{\x}$, where the tempered prior is defined as the marginal
$$ \ptt{\x} = \int \ptt{\x,\z} d\z $$
of the joint tempered prior defined in equation~\eqref{eq:hvae_joint_model_temp}.

This is shown in the next section.

\subsection{Fixed points are critical points}\label{app:fixed_points_are_critical_points}

In this section we characterize fixed points of Algorithm~\ref{algo:final} as critical points of a posterior density (a necessary condition to be a MAP estimator), under mild conditions. Before we formulate this caracterization we need to review in more detail a few facts about HVAE training, temperature scaling and our optimization model.

\paragraph{HVAE training.}
In section 3.1 we introduced how VAEs in general (and HVAEs in particular) are trained. As a consequence an HVAE embeds a joint prior
\begin{equation}  \label{eq:joint_prior_decoder}
	\pt{\x,\z} := \pt{\x|\z}\pt{\z} 
\end{equation}
from which we can define a marginal prior on $\x$
\begin{equation}  \label{eq:x-prior}
	\pt{\x} := \int \pt{\x,\z} d\z. 
\end{equation}

In addition, from the ELBO maximization condition in \eqref{eq:perfect_vae} and Bayes theorem we can obtain an alternative expression for the joint prior, namely
\begin{equation} \label{eq:joint_prior_encoder}
	\pt{\x,\z}  = \qp{\z|\x} \pd{\x}. 
\end{equation}

\paragraph{Temperature scaling.}
After training we reduce the temperature by a factor $\tauvec$, which amounts to replacing $\pt{\z}$ by 
$$ \ptt{\z} := \prod_{l=0}^{L-1} \frac{\pt{\zl|\z_{<l}}^{\frac{1}{\tau_l^2}}}{Z_l} $$
as shown in equation~\eqref{eq:hvae_joint_model_temp}, leading to the joint tempered prior
\begin{equation}\label{eq:joint_tprior_decoder}
	\ptt{\x,\z} :=  \pt{\x|\z}\ptt{\z}.
\end{equation}
The corresponding marginal tempered prior on $\x$ becomes
\begin{equation}  \label{eq:x-tprior}
	\ptt{\x} := \int \ptt{\x,\z} d\z
\end{equation}
and the corresponding posterior is
\begin{equation}  \label{eq:tposterior-decoder}
	\ptt{\z|\x} := \ptt{\x,\z} / \ptt{\x}.
\end{equation}

The joint tempered prior also has an alternative expression (based on the encoder). Indeed substituting $\pt{\x|\z}$ from equations~\eqref{eq:joint_prior_decoder}~and~\eqref{eq:joint_prior_encoder} into \eqref{eq:joint_tprior_decoder} we obtain
\begin{equation}\label{eq:joint_tprior_encoder}
	\ptt{\x,\z} = \frac{\ptt{\z}}{\pt{\z}} \qp{\z|\x}\pd{\x}.
\end{equation}
Substituting this result into definition~\eqref{eq:tposterior-decoder} we obtain an alternative expression for the tempered posterior
\begin{equation}
	\ptt{\z|\x} = \qp{\z|\x}\pd{\x}/\pt{\z}.
\end{equation}

\paragraph{Optimization model.}
Since we are using a scaled prior $\ptt{\x}$ encoded in our HVAE to regularize the inverse problem, the ideal optimization objective we would like to minimize is
\begin{equation}\label{eq:xMAP}
	U(\x) := \underbrace{-\log p(\y|\x)}_{f(\x)} \underbrace{- \log \ptt{\x}}_{g(\x)}.
\end{equation}

Since $\ptt{\x}$ is intractable our algorithm seeks to minimize a relaxed objective (see equation~\eqref{eq:J1}). Nevertheless, under certain conditions (to be specified below) this is equivalent to minimizing the ideal objective \eqref{eq:xMAP}.

\paragraph{Fixed-point characterization.}
We start by characterizing $\nabla \log \ptt{\x}$ in terms of an HVAE-related denoiser (Proposition~\ref{prop:tweedie}). Then we relate this denoiser to the quantity $\HVAE(\x,\tau)$ that is computed by our algorithm (Proposition~\ref{prop:denoiser-caracterization}). As a consequence we obtain that the fixed point condition in Proposition~\ref{prop:fixed_point} can be written as $\nabla U(\x)=0$ (see Corollary~\ref{cor:critical-point}).

\begin{proposition}[Tweedie's formula for HVAEs.]\label{prop:tweedie}
For an HVAE with Gaussian decoder $\pt{\x|\z} = \N(\x;\mut{\z},\gamma^2 I)$, the following denoiser based on the HVAE with tempered prior
\begin{equation}\label{eq:HVAE-denoiser}
	\Dtt{\x} := \int \mut{\z} \ptt{\z|\x} d\z
\end{equation}
satisfies Tweeedie's formula
\begin{equation}\label{eq:tweedie}
	\Dtt{\x} - \x = \gamma^2 \nabla \log \ptt{x} = - \gamma^2 \nabla g(\x).
\end{equation}
\end{proposition}

\begin{proof}
From the definition of $\ptt{\x}$ in equation~\eqref{eq:x-tprior} we have that
$$ \nabla\log\ptt{\x} = \frac{1}{\ptt{\x}}\int  \nabla_{\x} \pt{\x|\z} \ptt{\z} d\z.$$
From the pdf of the Gaussian decoder $\pt{\x|\z}$ its gradient writes
$$ \nabla_{\x} \pt{\x|\z}  = - \frac{1}{\gamma^2} (\x-\mut{\z})\pt{\x|\z}.$$
Replacing this in the previous equation we get
\begin{align*}
	\nabla\log\ptt{\x} & = \frac{1}{\gamma^2} \int (\mut{\z}-\x) \frac{\pt{\x|\z}\ptt{\z}}{\ptt{\x}}d\z \\
	& =  \frac{1}{\gamma^2} \int (\mut{\z}-\x) \ptt{\z|\x}d\z \\
	& =  \frac{1}{\gamma^2} \left( \int \mut{\z}\ptt{\z|\x}d\z - \x\right).
\end{align*}
In the second step we used the definitions of the joint tempered prior $\ptt{\x,\z}$~\eqref{eq:joint_tprior_decoder} and the tempered posterior $\ptt{\z|\x}$~\eqref{eq:tposterior-decoder}.
The last step follows from the fact that $\int \ptt{\z|\x} d\z = 1$ according to definitions~\eqref{eq:tposterior-decoder}~and~\eqref{eq:x-tprior}.
Finally applying the definition of the denoiser $\Dtt{\x}$ in the last expression we obtain Tweedie's formula~\eqref{eq:tweedie}.
\end{proof}

Under suitable assumptions the denoiser defined above coincides with $\HVAE(\x,\tauvec)$ computed by our algorithm.

\begin{assumption}[Deterministic encoder]\label{hyp:deterministic-encoder}
	The covariance matrices of the encoder defined in equation~\eqref{eq:gaussian_hvae} are 0, \emph{i.e.} $\Sigma_{\phi,l}(\z_{<l},\x)=0$ for $l=0,\dots,L-1$.
	Put another way $\qp{\z|\x}=\delta_{E_{\tauvec}(\x)}(\z)$ is a Dirac centered at $E_{\tauvec}(\x)$.
\end{assumption}

\begin{proposition}\label{prop:denoiser-caracterization}
	Under Assumption~\ref{hyp:deterministic-encoder} the function $\HVAE(\x,\tauvec)$ computed by Algorithm~\ref{algo:final} coincides with the denoiser $\Dtt{\x}$ defined in equation~\eqref{eq:HVAE-denoiser}.
\end{proposition}

\begin{proof}
	$\HVAE(\x,\tauvec)$ is defined in Proposition~\ref{prop:fixed_point} as
	$$ \HVAE(\x,\tauvec) = \mut{E_\tau(\x)}.$$
	
	First observe that for a deterministic encoder we also have $\ptt{\z|\x} = \delta_{E_{\tauvec}(\x)}(\z)$.
	Indeed for any test function $h$:
	\begin{align*} 
		\int h(\z) \ptt{\z|\x} d\z & =  \int h(\z)\qp{\z|\x} \pd{\x} /\pt{\z} d\z \\
		& = h(E_{\tauvec}(\x))  \underbrace{\pd{\x} /\pt{E_{\tauvec}(\x)}}_{Z(\x)}.
	\end{align*}
	And the normalization constant $Z(\x)$ should be equal to 1 because $\int \ptt{\z|\x} d\z=Z(\x)=1$.
	Hence $\ptt{\z|\x} = \qp{\z|\x} = \delta_{E_{\tauvec}(\x)}(\z)$.
	
	Finally applying the definition of $\Dtt{\x}$ we obtain
	\begin{align*} 
		\Dtt{\x} & = \int \mut{\z} \ptt{\z|\x} d\z = \mut{E_{\tauvec}(\x)} \\
		& = \HVAE(\x,\tau).
	\end{align*}
\end{proof}

Combining Propositions~\ref{prop:denoiser-caracterization},~\ref{prop:fixed_point}~and~\ref{prop:tweedie} we obtain a new characterization of fixed points as critical points.

\begin{corollary}\label{cor:critical-point}
	Under Assumption~\ref{hyp:deterministic-encoder} $\x^*$ is a fixed point of $T$ if and only if
	\begin{equation}
		\nabla f(\x^*) + \nabla g(\x^*) = 0
	\end{equation}
   where $g(\x) = - \log \ptt{\x}$.
\end{corollary}

\begin{proof}
	From Proposition~\ref{prop:tweedie} we have that
	$$- \nabla g(\x) = \frac{1}{\gamma^2} \left(\Dtt{\x} - \x\right).$$
	From Proposition~\ref{prop:denoiser-caracterization} we have that (under Assumption~\ref{hyp:deterministic-encoder}) $\Dtt{\x} = \HVAE(\x,\tau)$. In combination with the previous result:
	$$- \nabla g(\x) = \frac{1}{\gamma^2} \left(\HVAE(\x,\tau) - \x\right).$$
	Finally, Proposition~\ref{prop:fixed_point} allows to conclude that
	$$ - \nabla g(\x) = \nabla f(\x).$$
\end{proof}

\section{Details on PatchVDVAE architecture}\label{details}
In this section, we provide additional details about the architecture of PatchVDVAE. Then, we present the choice of the hyperparameters used for the concurrent methods (presented  in section~\ref{sec:expe} of the main paper) .
\subsection{PatchVDVAE}
Figure~\ref{fig:patchvdvae_arch} provides a detailed overview of the structure of a PatchVDVAE network.
The architecture follows VDVAE model~\cite{child2021very},
except for the first top-down block, in which we replace the constant input by a latent variable sampled from a Gaussian distribution.
 The architecture presented in figure~\ref{fig:patchvdvae_arch}  illustrates the structure of HVAE networks, but the number of blocks is different to the PatchVDVAE network used in our experiments.
Our PatchVDVAE top-down path is composed of $L=30$ top-down blocks of increasing resolution.
The image features are upsampled using an unpooling layer every $5$ blocks.
The first unpooling layer performs a $\times 4$ upsampling, and the following unpooling layers perform $\times 2$ upsampling.
The dimension of the filters is $256$ in all blocks. 
In order to save computations in the residual blocks, the $3\times3$ convolutions are applied on features of reduced channel dimension (divided by $4$). 
$1\times 1$ convolutions are applied before and after the $3\times3$ convolutions to respectively reduce and increase the number of channels.
The latent variables $\zl$ are tensors of shape $12 \times H_l \times W_l$, where the resolution $H_l$, $W_l$ corresponds to the resolution of the corresponding top-down-block.
The bottom-up network structure is symmetric to the top-down network, with $5$ residual blocks for each scale, and pooling layers between each scale.

\begin{figure}[!ht]
    \begin{subfigure}[b]{0.4\linewidth}
        \includegraphics[width=\linewidth]{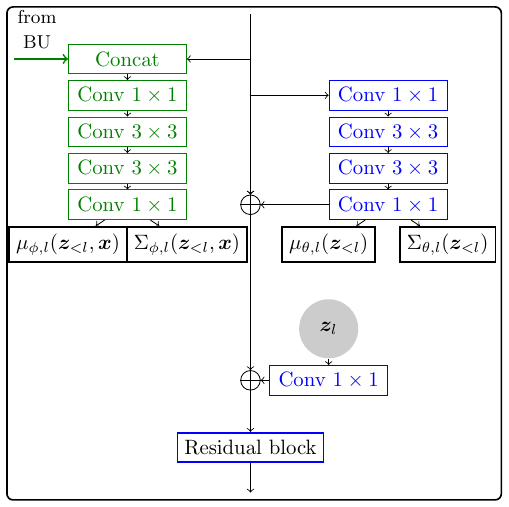}
        \subcaption*{Top-down block}
    \end{subfigure}
    \begin{subfigure}[b]{0.5\linewidth}
        \includegraphics[width=0.8\linewidth]{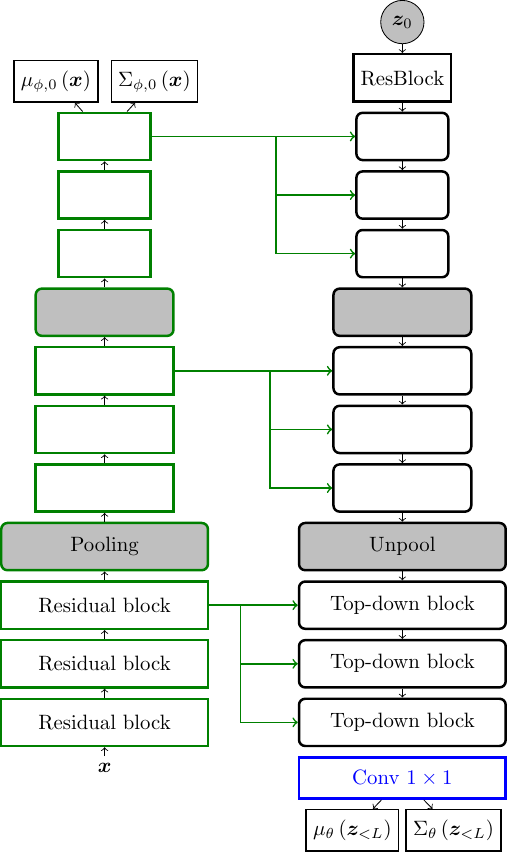}
        \subcaption*{Full autoencoder architecture}
    \end{subfigure}
    \begin{subfigure}[b]{0.4\linewidth}
        \centering
        \includegraphics[scale=1]{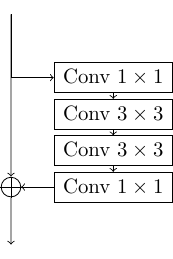}
        \subcaption*{Residual block}
    \end{subfigure}

    \caption{\label{fig:patchvdvae_arch}Structure of the PatchVDVAE architecture. For clarity, we omit the non-linearity after each convolution.}
\end{figure}

\subsection{Hyperparameters of compared methods}
\paragraph{Face image restoration.}
For ILO, we found that optimizing the first 5 layers of the generative network offered the best trade-off between image quality and consistency with the observation.
Hence, we optimize the 5 first layers for 100 iterations each. 
This choice is different from the official implementation, where they only optimize the 4 first layers for a lower number of iterations, trading restoration performance for speed.
For DPS, we set the scale hyper-parameter $\zeta'$ (described in subsection C.2 in\cite{chung2022diffusion})  to $\zeta'=1$ for the deblurring and super-resolution experiments reported in this paper.

\paragraph{Natural images restoration - Deblurring.}
For the three tested methods, we use the official implementation provided by the authors, along with the pretrained models.
For EPLL, we use the default parameters in the official implementation.

For GS-PnP, using the notation of the paper, we use the suggested hyperparameter $\lambda_\nu=0.1$ for the motion blur kernels %
and $\lambda_\nu=0.75$ for the Gaussian kernels.

For PnP-MMO, we use the denoiser trained on $\sigma_{den}=0.007$. 
On deblurring with $\sigma=2.55$ we use the default parameters in the implementation. 
for higher noise levels ($\sigma=7.65$; $\sigma=12.75$), and we  set the strength of the gradient step as $\gamma=\sigma_{den}/(2\sigma||h||)$, where $h$ corresponds to the blur kernel.

\paragraph{Natural images restoration- Inpainting.}
For EPLL, we use the default parameters provided in the authors matlab code.
For GS-PnP, after a grid-search, we chose to set $\lambda_\nu=1$ and $\sigma_{denoiser}=10$.

\section{Discussion on the conctractivity of HVAE}\label{sec:contractivity}

We showed in section~\ref{sec:convergence} that PnP-HVAE converges to a fixed point under the assumption that $\x \rightarrow \HVAE(\x, \tau)$ is contractive. If this condition is met, the sequence of $u_k$ defined by $u_{k+1} = HVAE(u_k, \tau)$ should converge to a fixed point.
Figure ~\ref{fig:fixed} presents the \textbf{evolution of a fixed point iteration} $u_{k+1} = HVAE(u_k, \tau)$. 
 The image is smoothened over the iterations, and finally converges to a piececewise constant image. We used patchVDVAE for this experiment.

\begin{figure}[h]
    \centering
   \includegraphics*[scale=1]{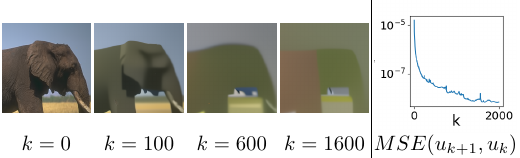}\vspace{-8pt}
   \caption{
      Fixed-point iterations of  patchVDVAE for $\tau=0.99$.\label{fig:fixed}} 
\end{figure}

\section{Comparisons}\label{add}
In this section, we provide additional visual results on face images and natural images.
\subsection{Additional results on face image restoration}
We provide additional comparisons with the GAN-based ILO method on inpainting (figure~\ref{fig:inp_add}), $\times 4$ super-resolution (figure~\ref{fig:sr_add}) and deblurring (figure~\ref{fig:deb_add}). PnP-HVAE provides equally or more plausible glasses in the first column) inpaiting than ILO. For superresolution, ILO produces sharper but not realistics faces. This is an agreement with the scores presented in table~\ref{table:comp_ILO}). For deblurring, ILO  creates textures on faces that looks realistic (low LPIPS) but are less consistent with the observation (significantly lower PSNR and SSIM).
\begin{figure}
    \includegraphics[width=\linewidth]{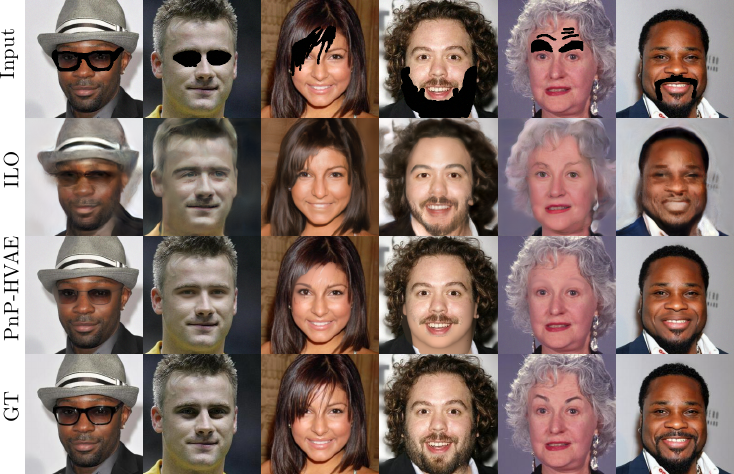}
    \caption{\label{fig:inp_add}Inpainting}
\end{figure}

\begin{figure}
    \includegraphics[width=\linewidth]{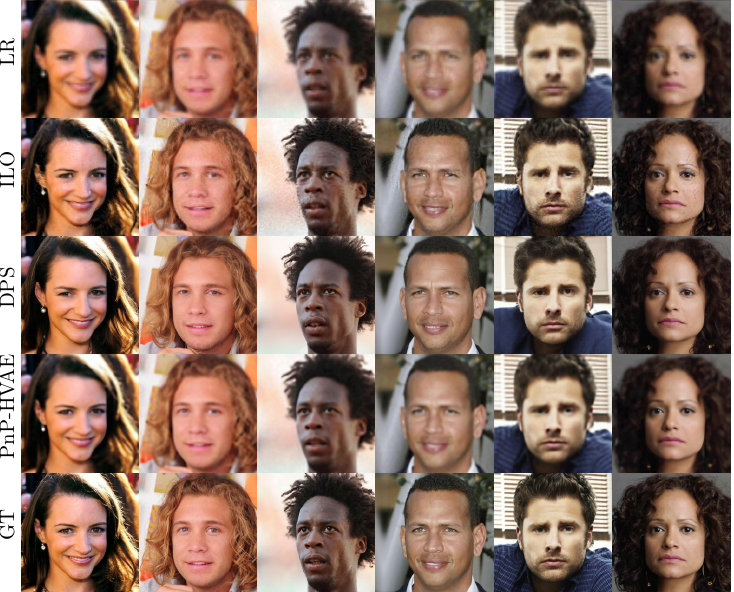}
    \caption{\label{fig:sr_add}$\times 4$ super-resolution, with kernel (a) from Figure~\ref{fig:kernels} and $\sigma=3$}
\end{figure}

\begin{figure}
    \includegraphics[width=\linewidth]{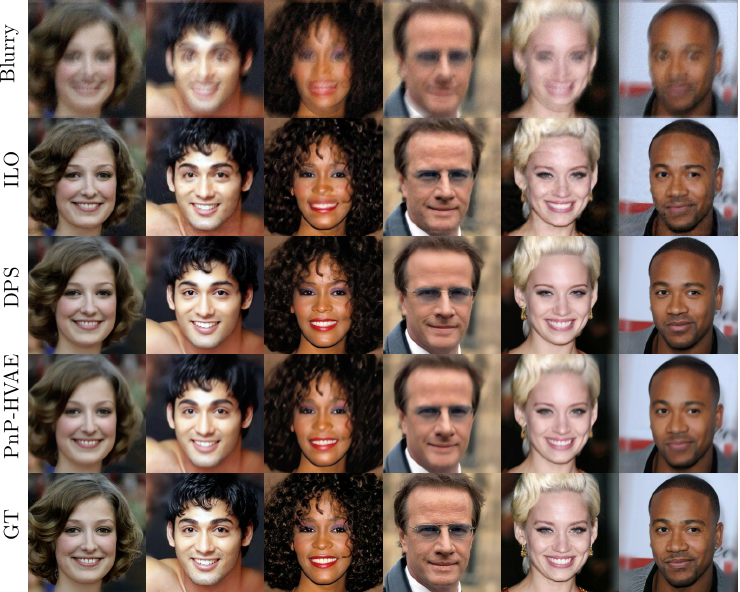}
    \caption{\label{fig:deb_add}Deblurring, with kernel (d) from Figure~\ref{fig:kernels} and $\sigma=8$}
\end{figure}

\subsection{Additional results on natural images restoration}
We finally present additional results on natural images restoration.
All the PnP-HVAE images presented below were produced using our PatchVDVAE model.
We also provide visual comparisons with concurrent PnP methods and EPLL. For deblurring (figures~\ref{fig:deb1} and~\ref{fig:deb2}, PnP methods perform better than EPLL. Following quantitative results of figure~\ref{tab:deb}, for larger noise level, PnP-HVAE outperforms PnP-MMO and provides restoration close to GS-PnP.

For inpainting (figure~\ref{fig:demo_inp}), the hierarchical structure of PatchVDVAE leads to more plausible reconstructions,  and PnP-HVAE outperforms the compared methods. \vspace{0.05cm}

\begin{figure}
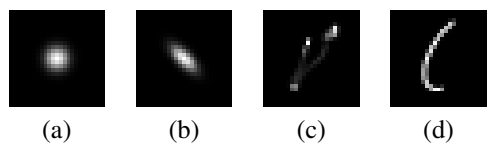

    \centering
    \begin{tabular}[]{c c c c}
        \includegraphics[scale=2]{figures/kernels/4.png} &
        \includegraphics[scale=2]{figures/kernels/7.png} &
        \includegraphics[scale=2]{figures/kernels/9.png} &
        \includegraphics[scale=2]{figures/kernels/11.png} \\
        (a) & (b) & (c) & (d) 
    \end{tabular}
    \caption{\label{fig:kernels}Kernels used for deblurring experiments, from~\cite{levin2009understanding}}
\end{figure}

\newlength{\fwidth}
\setlength{\fwidth}{0.9\linewidth}
\begin{figure}
    \begin{subfigure}[h]{\linewidth}
        \centering
    
        \includegraphics[width=\fwidth]{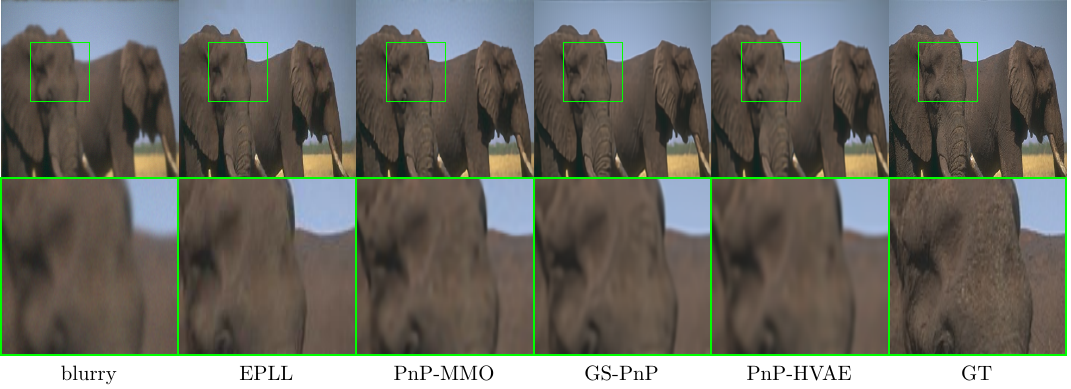}
        \caption{kernel (a), $\sigma=2.55$}
    \end{subfigure}
    
    \begin{subfigure}[h]{\linewidth}
    \centering
    \includegraphics[width=\fwidth]{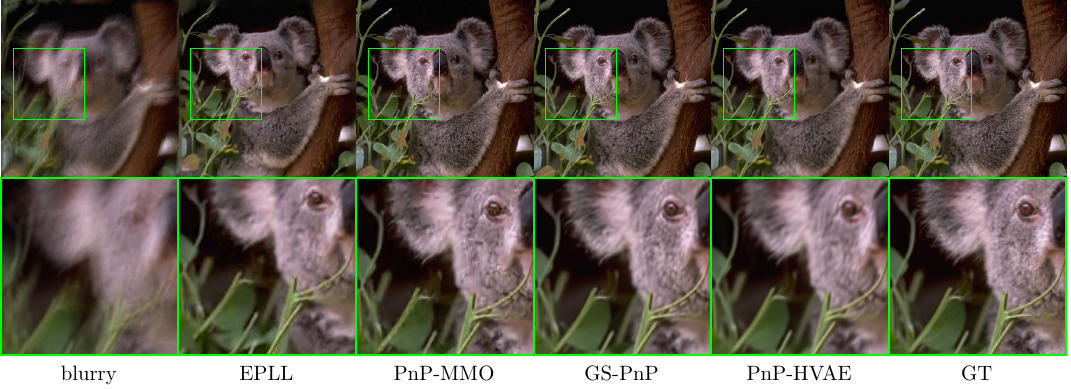}
    \caption{kernel (c), $\sigma=2.55$}
    \end{subfigure}

    \begin{subfigure}[h]{\linewidth}
        \centering
        \includegraphics[width=\fwidth]{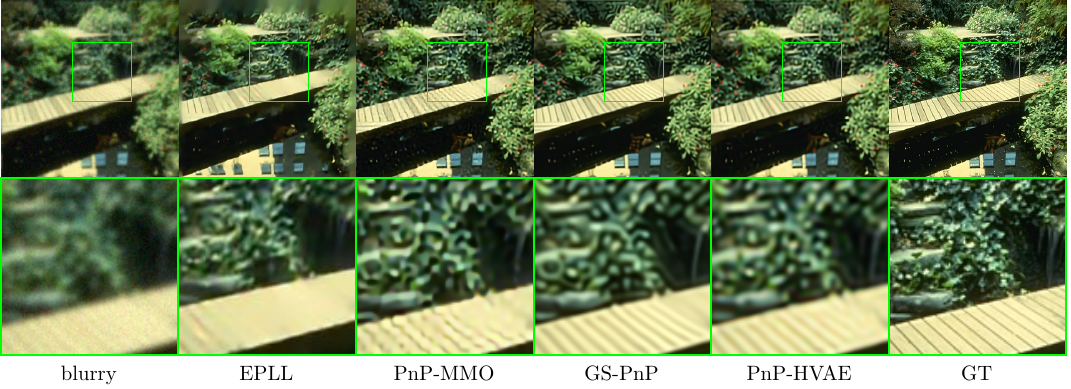}
        \caption{kernel (a), $\sigma=7.65$}
    \end{subfigure}
    \caption{\label{fig:deb1}Deblurring results on BSD}
\end{figure}

\begin{figure}
    \begin{subfigure}[h]{\linewidth}
    \centering
    \includegraphics[width=\fwidth]{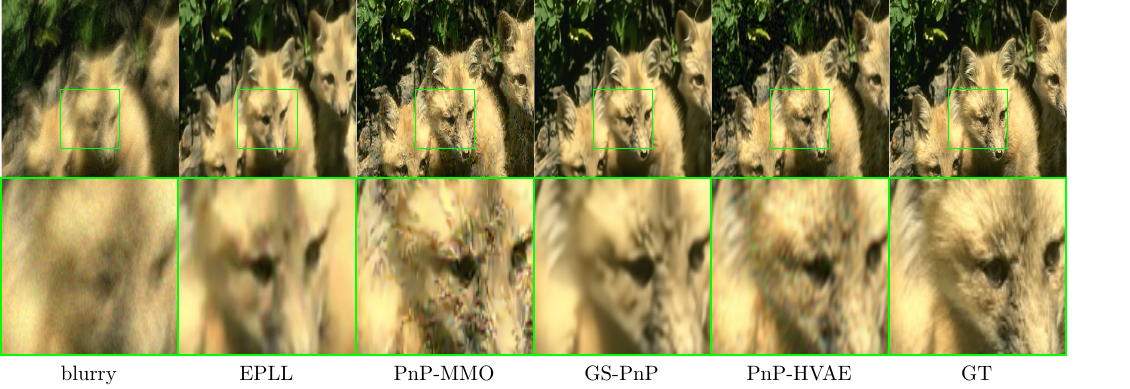}
    \caption{kernel (d), $\sigma=7.65$}
    \end{subfigure}

    \begin{subfigure}[h]{\linewidth}
    \centering
    \includegraphics[width=\fwidth]{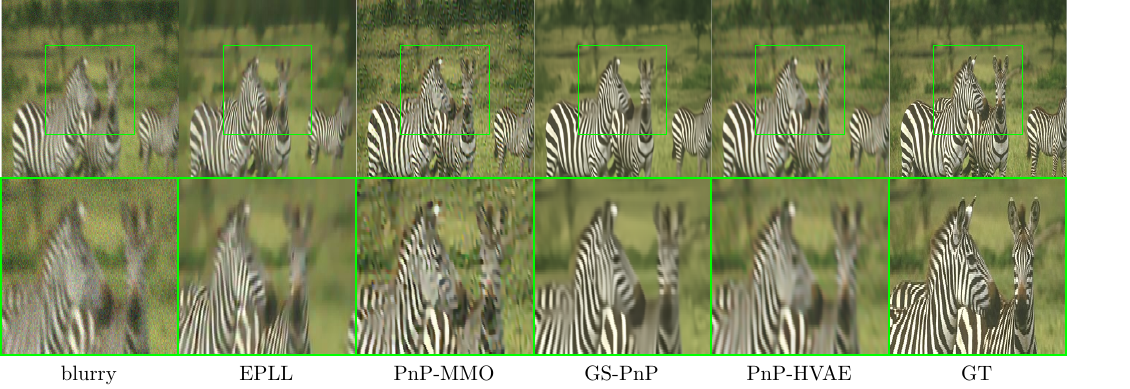}
    \caption{kernel (b), $\sigma=12.75$}
    \end{subfigure}

    \begin{subfigure}[h]{\linewidth}
    \centering
    \includegraphics[width=\fwidth]{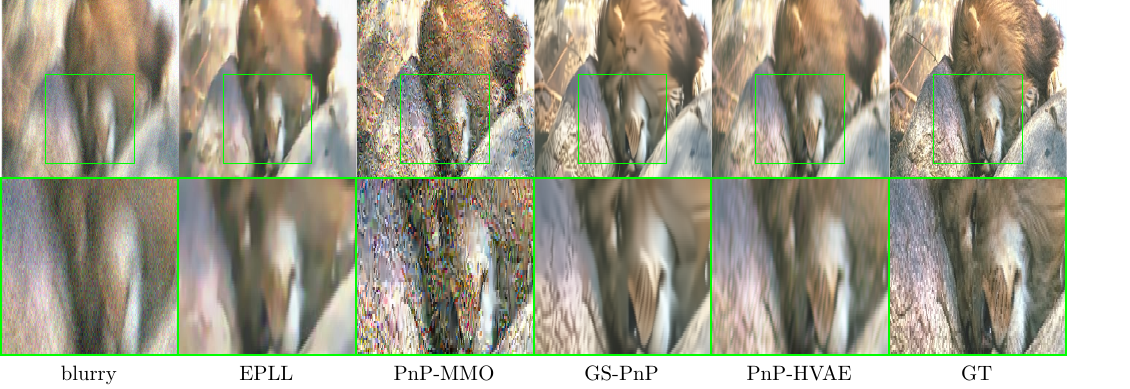}
    \caption{kernel (d), $\sigma=12.75$}
    \end{subfigure}
    \caption{\label{fig:deb2}Deblurring results on BSD}
\end{figure}

\begin{figure}
    \centering
    
    \includegraphics{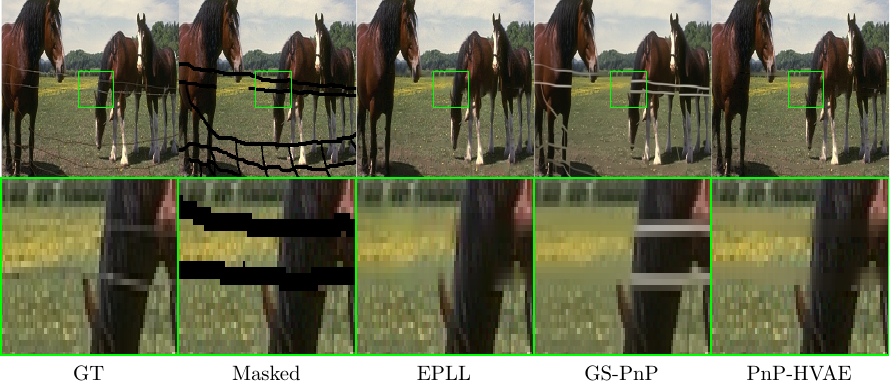}
    \caption{\label{fig:demo_inp}Natural images inpainting}
\end{figure}

\end{document}